\documentclass[journal,draftcls,onecolumn,11pt,twoside]{IEEEtran}

\usepackage[utf8]{inputenc} 
\usepackage[T1]{fontenc}    
\usepackage{hyperref}       
\usepackage{url}            
\usepackage{booktabs}       
\usepackage{amsfonts}       
\usepackage{nicefrac}       
\usepackage{microtype}      
\usepackage{comment}
\usepackage{amsthm}
\usepackage{pgfplots}
\usepackage{amsmath}
\usepackage{algorithm}
\usepackage{algpseudocode}
\usepackage{amssymb}

\newcommand\norm[1]{\lVert#1\rVert}

\newcommand{\txtt}{\Tilde{x}_{t+1}}

\newcommand{\stt}{\alpha_t}
\newcommand{\td}[1]{\Tilde{#1}}
\newcommand{\hxt}{\hat{x}_{t}}
\newcommand{\hxtt}{\hat{x}_{t+1}}
\newtheorem{theorem}{Theorem}
\newtheorem{lemma}[theorem]{Lemma}
\newtheorem{definition}{Definition}
\newtheorem{remark}{Remark}
\usepackage{subcaption}
\def\Figurespath{./log_figures}

\def\Figurespath{./log_figures}

\title{Adaptive Step-Size Methods for Compressed SGD}

\author{\begin{tabular}{c} Adarsh M. Subramaniam,
Akshayaa Magesh,
Venugopal V. Veeravalli
\end{tabular} \\
Department of Electrical and Computer Engineering, \\
University of Illinois at Urbana-Champaign
}

\begin{document}

\maketitle

\begin{abstract}
  Compressed Stochastic Gradient Descent (SGD) algorithms have been recently proposed to address the communication bottleneck in distributed and decentralized optimization problems, such as those that arise in federated machine learning.  Existing compressed SGD algorithms assume the use of non-adaptive step-sizes~(constant or diminishing) to provide theoretical convergence guarantees. Typically, the step-sizes are fine-tuned in practice to the dataset and the learning algorithm to provide good empirical performance. Such fine-tuning might be impractical in many learning scenarios, and it is therefore of interest to study compressed SGD using adaptive step-sizes. Motivated by prior work on adaptive step-size methods for SGD to train neural networks efficiently in the uncompressed setting, we develop an adaptive step-size method for compressed SGD.
  In particular, we introduce a \emph{scaling} technique for the descent step in compressed SGD, which we use to establish order-optimal convergence rates for convex-smooth and strong convex-smooth objectives under an \textit{interpolation condition} and for non-convex objectives under a \textit{strong growth condition}. We also show through simulation examples that without this scaling, the algorithm can fail to converge. We present experimental results on deep neural networks for real-world datasets, and compare the performance of our proposed algorithm with previously proposed compressed SGD methods in literature, and demonstrate improved performance on ResNet-18, ResNet-34 and DenseNet architectures for CIFAR-100 and CIFAR-10 datasets at various levels of compression.
\end{abstract}

\section{Introduction}

Consider the optimization problem of minimizing the sum of $n$ functions :
\begin{equation}\label{eq:min_func}
\min_{x} f(x) = \min_{x} \frac{1}{n}\sum_{i=1}^{n} f_i(x).
\end{equation}
This formulation is widely used in machine learning, where $f_i$ is the loss function corresponding to datapoint $i$, $n$ is the size of the training set, and $x$ corresponds to the parameters of the learning model.  

A practical approach to solving the optimization problem in \eqref{eq:min_func} is stochastic gradient descent~(SGD) where one of the component functions $f_i$ is sampled at random at a given iteration step, and the iterate $x$ is updated using the gradient of $f_i$. The main problem of interest in this paper is one where compressed versions of the gradients are used in the SGD iterations. The use of compressed gradients allows for the alleviation of the communication bottleneck in distributed and decentralized optimization settings such as those that arise in federated learning \cite{kairouz2021advances}.

There have been several recent papers on compressed SGD algorithms (see, e.g., \cite{dryden_quant,aji_hea,stich2018sparsified,alistarh2018convergence,Tang2018CommunicationCF,sign_sgd,basu_qsparselocalsgd,pmlr-v130-shokri-ghadikolaei21a,vqSGDVQ,Kovalev2021ALC,lin2021quasiglobal}).   Compressed gradient transmission was first explored in \cite{dryden_quant,aji_hea}, where only the top $k$~(e.g., $k= 1\%$) values of the components of the gradient (with feedback) are used in the iteration update. 
The convergence properties of such compressed gradient methods with feedback were studied in a number of subsequent papers (e.g.,  \cite{alistarh2018convergence,stich2018sparsified}). Feedback has also been shown to fix convergence issues in compression algorithms such as Sign-SGD \cite{karimireddy2019error}. 
In the decentralized setting, the convergence of SGD with compression and the use of gossip algorithms has been established  in~\cite{koloskova2019decentralized}. 

Existing works on compressed SGD algorithms assume the use of  diminishing or constant step-sizes based on system parameters to show convergence properties. In applications to machine learning, the step-sizes are fine tuned to the dataset and the machine learning model to provide good empirical performance. However, such fine tuning can be impractical and can result in sub-optimal rates of convergence. Hence it is of interest to study adaptive step-size methods for compressed SGD.  




Adaptive step-size techniques for (uncompressed) SGD have been studied in several recent works~(see, e.g.,~\cite{stoc_painless,duboistaine2021svrg,vaswani2021adaptive}). An efficient implementation of the classical Armijo step-search method~\cite{armijo_ls} to train neural networks has been proposed in the work of~\cite{stoc_painless}. Extensions to improve the performance of adaptive gradient methods such as Adagrad~\cite{adagrad} and AMSgrad~\cite{ams_grad} have been proposed in~\cite{vaswani2021adaptive}. Adaptive step-size methods for variance reduction algorithms like SVRG have also been proposed in~\cite{duboistaine2021svrg}. In order to establish convergence, these works assume that the objective satisfies an \textit{interpolation condition}. The condition essentially implies that there exists a point $x^*$ which is a minimum for all the functions $f_i$ in the optimization problem~\eqref{eq:min_func}. 
Several other recent works,~\cite{bassily2018exponential,Liu2018AcceleratingST,ma2018power} have used the \textit{interpolation condition} to establish convergence of non-compressed SGD in various settings. 
The use of this condition has also been supported by theoretical works  \cite{arora2019finegrained,montanari2021interpolation}, in the context of modern neural networks, where the number of parameters is much larger than the number of datapoints in the training dataset. The interpolation condition is also satisfied by models such as  non-parameteric regression~\cite{Liang_2020}, boosting~\cite{boosting_schapire} and linear classification on separable data \cite{hastie_stat}.

Motivated by adaptive step-size techniques for uncompressed SGD, we develop an adaptive step-size technique for compressed SGD that results in faster convergence at the same compression level than existing non-adaptive step-size compressed SGD algorithms. To this end, we develop a \textit{scaling technique}, which is crucial to the convergence of our adaptive step-size compressed SGD algorithm. To the best of our knowledge, we are not aware of any adaptive step-size SGD methods for compressed SGD with feedback that have theoretical guarantees.

\subsection*{Our Contributions}
\begin{enumerate}

\item We introduce the idea of scaling for the classical Armijo step-size search, and demonstrate that it can accelerate convergence in certain asymmetric optimization problems. 

\item {Using a scaling technique, we prove that the Armijo step-size search \cite{armijo_ls} for Gradient Descent~(GD) converges at the rate of $O(\frac{1}{T})$ for convex and smooth functions, where $T$ is the number of iterations,  for all values of $\sigma \in (0,1)$}. Here $\sigma$ is a parameter of Armijo's rule (see \eqref{eq:armijo_ls} for a definition). To the best of our understanding, prior works have been able to show $O(\frac{1}{T})$ convergence only for values of $\sigma\geq 0.5$ .

\item  We propose a computationally feasible and efficient adaptive step-size search compressed SGD algorithm that uses the scaling technique and the biased $top_k$ operator for compression. We prove under an \textit{interpolation condition} (see Definition~\ref{def:interpol}), our proposed algorithm achieves a convergence rate of $O(\frac{1}{T})$ when the objective function is convex and smooth, and $O(\beta^T)$ ($\beta < 1$) when the objective function is strongly-convex and smooth. We also show that under a \textit{strong growth condition} (see Definition~\ref{def:sgc}), the algorithm achieves a convergence rate of $O(\frac{1}{T})$ in the non-convex setting.  The above results hold for all values of the parameter $\sigma \in (0,1)$ of the step-search. 

\item We demonstrate through experiments that the scaling technique is not just a proof technicality; without scaling, a direct implementation of Armijo step-size search for compressed SGD can result in divergence of the algorithm from the minimum.

\item  We demonstrate that our algorithm for compressed SGD outperforms existing non-adaptive SGD methods on neural network training tasks at various levels of compression.

\end{enumerate}


\section{Preliminaries and Notation}

In this section, we lay out the notation used in this paper, discuss the classical Armijo step-size search, the compression operator $top_k$,  and the conditions under which we prove our convergence results.

\subsection{Notation}
Let $x_t$ denote the iterate of the optimization algorithm at time $t$, and $i_t$ denote the datapoint or batch of data chosen at time $t$. The loss function of the batch $i_t$ chosen at time $t$ at an iterate $x_t$ is denoted by $f_{i_t}(x_t)$. The step-size used at time $t$ is denoted by $\eta_t$. 

\subsection{Armijo Step-Size Search in Gradient Descent} \label{sec:Armijo_GD}
The algorithm proposed in this work makes use of a modified and efficient implementation of the classical Armijo step-size search method~\cite{armijo_ls}. For a function $f$, the method at iteration $t$ searches for the step-size starting at $\alpha_{\mathrm{max}}$, and decreases the step-size by a factor $\rho<1$ until the condition
\begin{equation}\label{eq:armijo_ls}
 f(x_t -  \alpha_t \nabla f(x_t)) \leq f(x_t) - \sigma \alpha_t ||\nabla f(x_t)||^2
\end{equation}
is met, where $\sigma$ is a parameter of the algorithm. The step-size $\alpha_t$ that satisfies \eqref{eq:armijo_ls} is substituted as the value of the step-size in the iteration of the GD algorithm. The pseudocode for Armijo step-size search is given in  Algorithm~\ref{alg:armijo_ls}.

\begin{algorithm} 
\caption{Armijo Step-Size Search}
\label{alg:armijo_ls}
\begin{algorithmic}[1]
\Procedure{armijo step-size search}{$f,\alpha_{\mathrm{max}},x_t,t$}       
    \State $\alpha_t = \alpha_{\mathrm{max}}$
        \Repeat
        \State $\alpha_t \leftarrow \alpha_t\rho$ 
        \State $\tilde{x}_{t+1} \leftarrow x_{t} - \alpha_t  \nabla f(x_t)$ \label{step:armijo_line}
        \Until{$f(\tilde{x}_{t+1}) \leq f(x_{t}) - \sigma\alpha_t ||\nabla f(x_t)||^2$}\label{step:armi_stop} \\
    \hspace{0.025\textwidth} \Return $\alpha_t$
\EndProcedure
\end{algorithmic}
\end{algorithm}

\subsection{Top\textsubscript{k} Compression Operator}

In our compressed SGD algorithm, the gradient in each iteration is compressed using the $top_k$ compression operator (see, e.g., \cite{stich2018sparsified,basu_qsparselocalsgd}). Let $x \in \mathbb{R}^d$. Let $\mathsf{T_{k}}$ be the set of indices of the $k$ elements of $x$ with maximum magnitude, and define $\gamma \triangleq \frac{k}{d}$. The $top_k$ operator compresses a vector $x$ such that 
\begin{equation}
top_k(x)_i = \begin{cases}
(x)_i & \text{If $i\in\ \mathsf{T_{k}}$} \\
0 & \text{otherwise}\\
\end{cases}.
\end{equation}

\subsection{Definitions}

To establish the convergence of our algorithm, we require the following interpolation condition in the convex setting, and the strong growth condition in the non-convex setting.

\begin{definition}[Interpolation]\cite{liu2019accelerating}
A function $f(x) = \frac{1}{n} \sum_{i=1}^{n}f_i(x)$ satisfies the interpolation condition if 
\begin{equation}\label{eq:interpol}
 \exists \ x^* \;\;\; \text{s.t.} \;\;\; \nabla f_i(x^*) =0, \forall  i\in \{1,2,\cdots,n\}.   
\end{equation}
\label{def:interpol}
\end{definition} 
\begin{definition}[Strong Growth Condition]\cite{stoc_painless}
A function $f= \frac{1}{n} \sum_{i=1}^{n}f_i(x)$ satisfies the strong growth condition with constant $\nu$ if 
\begin{equation}\label{eq:sgc}
    \frac{1}{n}\sum_{i=1}^{n}||\nabla f_i(x)||^2 \leq \nu ||\nabla f(x)||^2
\end{equation}
\label{def:sgc}
\end{definition}

\section{Compressed SGD with Adaptive Step-Size}\label{sec:biased_comp}

Our proposed compressed SGD algorithm is motivated by the idea of scaling, where the step-size returned by the Armijo rule is scaled by a factor $a$ in the descent step. We now discuss the intuition for such scaling, and propose an algorithm for compressed SGD with Armijo step-size search and scaling.

\subsection{Armijo Step-Size Search with Scaling}
Optimization problems that arise in machine learning applications are usually asymmetric, and have gradients that do not point in the direction of a minimizer. For example, consider the function $f(x) = \frac{x_1^2}{4} + \frac{x_2^2}{9}$. The negative gradient at a point is orthogonal to the tangent and does not {necessarily} point towards the minimizer $(0,0)$. Hence, GD with Armijo step-size search may suffer from slower convergence if $\alpha_{\mathrm{max}}$ is large. In order to alleviate this issue, we introduce the idea of scaling where the step-size $\alpha_t$ returned by  Armijo's rule \eqref{eq:armijo_ls} is scaled by a factor $a< 2\sigma$, and $\eta_t \triangleq a \alpha_t$ is used as the step-size in the descent step, i.e.,
\[
x_{t+1} = x_t -  \eta_t \nabla f(x_t).
\]
It is important to note that the stopping criterion for Armijo's rule  \eqref{eq:armijo_ls} uses $\alpha_t$ (and not $\eta_t$) to scale the gradient. In the Appendix~(Figure~\ref{fig:scale_plots}), we plot the loss for GD with Armijo step-size search for the asymmetric function 
$f(x) = \sum_{i=1}^{10} \frac{x_i^2}{2^i}$, and show that the scaled GD algorithm outperforms the non-scaled one by several orders of magnitude in terms of convergence time to the minimum.


Another aspect of GD algorithm with Armijo step-size search is that while the algorithm converges for all values of $\sigma\in(0,1)$, a convergence rate of $O(1/T)$ for convex objectives (where $T$ is the number of iterations of the GD algorithm) has been proven only for values of $\sigma \geq 0.5$ \cite{armijo_ls}. On the other hand, commonly used values of $\sigma$ in practice are in the range $[10^{-5} ,\ 10^{-1}]$ \cite{bertsekas_book}.  With scaling added to Armijo step-size search, we prove a convergence rate of $O(\frac{1}{T})$ for convex objectives for all values of $\sigma \in (0,1)$. To the best of our knowledge, this is the first such proof, and is presented in Section~\ref{sec:det_gd} of the Appendix. 
\subsection{Compressed SGD with Armijo Step-Size Search and Scaling (CSGD-ASSS)}
Following \cite{aji_hea}, the method that we used for gradient compression is one that utilizes memory feedback to compensate for the error due to compression. The corresponding update rule for the iterate at time $t$ is given by:
\begin{equation}
\begin{aligned}
x_{t+1} = x_t - top_k(m_t + \eta_t \nabla f_{i_t}(x_t))
\end{aligned}    
\end{equation}
where $\eta_t$ is the step-size, $i_t$ is the index of the function sampled at time $t$, and $m_t$ is the error value due to compression that is updated as in step~\ref{step:error_update} of Algorithm~\ref{alg:biased_compressed_sgd}.

In the compressed SGD setting, even if the negative gradient $- \nabla f_{i_t} (x_t)$ points towards a minimizer of $f$,
the descent direction $-top_k(m_t + \eta_t \nabla f_{i_t}(x_t))$ may be skewed away from the minimizer. Hence step-size search based on Armijo's rule (without scaling) used in~\cite{stoc_painless} for uncompressed SGD may diverge (exponentially) for compressed SGD,  as we show in Section~\ref{subsec:scaling_fundamental_plots}. 
Motivated by our results on scaling for GD with Armijo step-size search, we propose scaling as a fix for compressed SGD with Armijo step-size search as discussed below.


At iteration $t$, the CSGD-ASSS algorithm implements Armijo step-size search~(Algorithm \ref{alg:armijo_ls}) with the  function $f_{i_t}$ and returns a value $\alpha_t$ that satisfies 
\begin{equation}
f_{i_t}(x_t - \alpha_t \nabla f_{i_t}(x_t)) \leq f_{i_t}(x_t) - \sigma \alpha_t ||f_{i_t}(x_t)||^2.
\end{equation}

The step-size at iteration $t$ is subsequently chosen as $\eta_t = a\alpha_t$, where $a$ is a scaling factor. The algorithm computes the compressed gradient with feedback error $m_t$ and scaled step-size $\eta_t$ as
\begin{equation}
    g_t = top_k(\eta_t \nabla f_{i_t}(x_t) + m_t).
\end{equation}

The iterate and error for iteration $t+1$ are updated, respectively, as 
\[
x_{t+1} = x_t - g_t
\]
and 
\[
m_{t+1} =\eta_t \nabla f_{i_t}(x_t) + m_t - top_k(\eta_t \nabla f_{i_t}(x_t) + m_t).
\]

\begin{algorithm}
\caption{Compressed SGD + Armijo Step-Size Search and Scaling~(CSGD-ASSS)}
\label{alg:biased_compressed_sgd}
\begin{algorithmic}[1]
\For{$t = 1,\cdots, T$}
    \State Sample batch $i_t$ of data
    \State $\alpha_{\mathrm{max}} = \omega\alpha_{t-1}$
    \State $\alpha_{t} \leftarrow \text{Armijo Step-Size Search}(f_{i_t},\alpha_{\mathrm{max}},x_t,t)$\label{step:ls_csgd}
    \State $\eta_t = a \alpha_t$
    \State $g_{t} = top_k(m_t + \eta_t \nabla f_{i_t}(x_t))$
    \State $x_{t+1} \leftarrow x_{t} -  g_{t} $ 
    \State $m_{t+1} = m_t + \eta_t \nabla f_{i_t}(x_t) - top_k(m_t + \eta_t \nabla f_{i_t}(x_t)) $\label{step:error_update}
\EndFor
\end{algorithmic}
\end{algorithm}
\begin{remark}
 The CSGD-ASSS algorithm in the distributed setting~(presented in the Appendix) adaptively tunes the step-size $\eta_t$ specifically to the data at each node.
\end{remark}

\subsection{Convergence Analysis}
We now study the convergence of the CSGD-ASSS algorithm in the convex and strong convex settings under the interpolation condition \eqref{eq:interpol}, and in the non-convex setting under the strong growth condition \eqref{def:sgc}. These rates achieved under compression are the same as those in the uncompressed setting studied in \cite{stoc_painless}.

\subsubsection{Main Results}

Consider the minimization of the objective $f(x) = \frac{1}{n} \sum_{i=1}^{n} f_i(x)$ with CSGD-ASSS algorithm. For the convex setting we prove the following convergence result under the interpolation condition.

\begin{theorem}[CSGD-ASSS convex] \label{thm:bcmf_convex}
Let $f_i(x)$ be convex and $L_i$ smooth for $i\in[n]$, and assume that the interpolation condition \eqref{eq:interpol} is satisfied. Then there exists $\hat{a} >0$ such that for  $0< a \leq \hat{a}$, the CSGD-ASSS algorithm with scaling $a$, parameter $\sigma \in (0,1)$, and compression factor $\gamma = \frac{k}{d}$,  satisfies
\begin{equation}
 E\left[f\left(\frac{1}{T} \sum_{t=0}^{T-1}x_t \right)\right] -E[f(x^*)] \leq  \frac{1}{\delta_1 T}  ( E[||x_0 - x^*||^2] )
\end{equation}
where for any $0< \epsilon < \zeta$, $\zeta \triangleq \frac{\sigma \gamma}{(2-\gamma)}$ and $L_{\mathrm{max}} \triangleq \max_{i}L_i$,

\begin{equation}
\begin{aligned}
&\delta_1 \triangleq \rho \frac{2(1-\sigma)}{L_{\mathrm{max}}} \Big(2 a  - \frac{a^2 }{\sigma} - \frac{a^2 }{\sigma p(\epsilon)}   - (1-\gamma)\left(1+ \frac{1}{r(\epsilon)}\right) \frac{a^2 }{\sigma}  \Big),\\
& \hat{a} = (\zeta-\epsilon). \\
\end{aligned}    
\end{equation}

\end{theorem}

The values of $p(\epsilon),r(\epsilon)$ are obtained from the following convex problem.

\begin{equation}\label{eq:p_r_def_problem}
\begin{aligned}
    \max_{p,r} \tilde{a}_1(p,r)  =  2 \sigma/ \Big(1+ \frac{1}{p} + (1-\gamma)(1+ \frac{1}{r}) \Big)\\
    \mathrm{s.t.} (p + (1-\gamma)(1+ r)) \leq 1\\
    p \geq 0, r \geq 0
\end{aligned}    
\end{equation}

Let the solution to the problem be $p_{\mathrm{max}},r_{\mathrm{max}}$. It is shown in the Appendix that the maximum value $\tilde{a}_1(p_{\mathrm{max}},r_{\mathrm{max}})$ is $\zeta $. Due to the continuity of $\tilde{a}_1(p,r)$, the parameters $p(\epsilon),r(\epsilon)$ can be chosen to be any values that satisfy:
\begin{equation}\label{eq:epsilon_ball}
\begin{aligned}
  &(p(\epsilon) + (1-\gamma)(1+ r(\epsilon))) < 1 \\
  & \tilde{a}_1(p(\epsilon),r(\epsilon)) > \zeta - \epsilon.
\end{aligned}
\end{equation}


\vspace{0.5em}

Consider the minimization of the function $f(x)$ in the convex setting, with some $f_i(x)$ satisfying strong convexity with parameter $\mu_i >0$. Then we can establish the following  `linear' convergence result.
\vspace{0.5em}
\begin{theorem}[CSGD-ASSS strong convex]\label{thm:strong_convex}
Let $f_i$ be convex and $L_i$ smooth for all $i\in[n]$, and $\mu_i$ strongly convex with $\mu_i >0$ for some $i\in[n]$. Assume that the interpolation condition \eqref{eq:interpol} is satisfied. Then there exists $\hat{a} >0$ such that for  $0< a \leq \hat{a}$, the CSGD-ASSS algorithm with scaling $a$ and $\sigma \in (0,1)$  satisfies
\begin{equation}
\begin{aligned}
E[||x_t - x^*||^2] \leq  2 (\hat{\beta})^t E[(||{x}_0 - x^*||^2  )],
\end{aligned}    
\end{equation}
for some $\hat{\beta}<1$, where for all $0<\epsilon< \zeta$
\begin{equation}
\begin{aligned}
&\hat{\beta} = \max{\{\beta_1(p{(\epsilon)}, r{(\epsilon)}),(1- \frac{E_{i_t}[\mu_{i_t}]a (1-\sigma)}{ L_{\mathrm{max}}})\}},\\
&\beta_1(p(\epsilon),r(\epsilon)) = \Big(\mu_{\mathrm{max}} a \alpha_{\mathrm{max}}  + p(\epsilon)  +(1-\gamma)(1+r(\epsilon)) \Big)\\
&\hat{a}=\zeta - \epsilon
\end{aligned}    
\end{equation}
\end{theorem}

The values $p(\epsilon),r(\epsilon)$ are as defined in \eqref{eq:p_r_def_problem} and \eqref{eq:epsilon_ball}, and $\mu_{\mathrm{max}}$ is defined in the Appendix.


To show convergence guarantees in the non-convex case, we use the strong growth condition \eqref{eq:sgc}.
\begin{theorem}[CSGD-ASSS non-convex]\label{thm:non_conv_asss} Let $f_i$ be convex, $L_i$ smooth and let $f_i$ $i\in[n]$ satisfy the strong growth condition \eqref{eq:sgc}. Then, there exists $\hat{a},\hat{\alpha}$ such that for $0< a\leq \hat{a}$ and $\alpha_{\mathrm{max}} \leq \hat{\alpha}$ , 

\begin{equation}
\begin{aligned}
\frac{1}{T} \sum_{t=0}^{T-1} E[\norm{\nabla f(x_t)}^2] \leq \frac{(E[f({x}_{0})] - E[f(\hat{x}_{T})])}{\delta T}  
\end{aligned}    
\end{equation}
\end{theorem}
where $\hat{x}_T$ is a perturbed iterate~\cite{Mania2017PerturbedIA} obtained from $\{x_i\}_{i=1}^{T}$. A discussion on the choice of $\hat{a}$, and $\hat{\alpha}$ is presented in the Appendix.

\begin{remark}
 Proofs for the uncompressed SGD setting~\cite{stoc_painless} constrain the values of both $\sigma$ and $ \alpha_{\mathrm{max}}$ in the non-convex setting, and the bounds are dependent on the function parameters $L_i$. Interestingly the scaling technique allows us to eliminate the bound on $\sigma$. 
\end{remark}

\section{Experimental Results}



\begin{figure*}%
\centering
\begin{subfigure}{.4\columnwidth}
\includegraphics[width=\columnwidth]{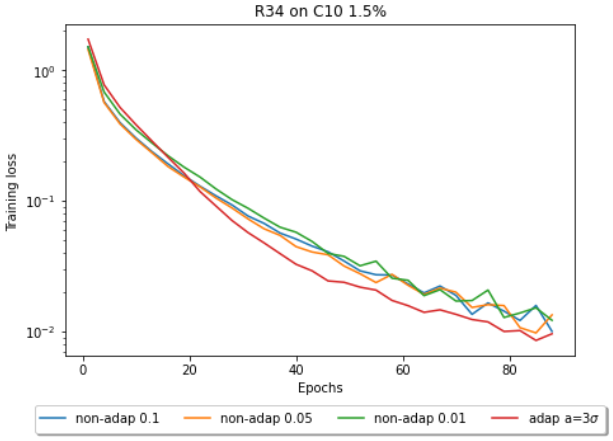}%
\caption{ResNet34,CIFAR10,1.5\% }%
\label{r34c10p1}%
\end{subfigure}\hfill%
\begin{subfigure}{.4\columnwidth}
\includegraphics[width=\columnwidth]{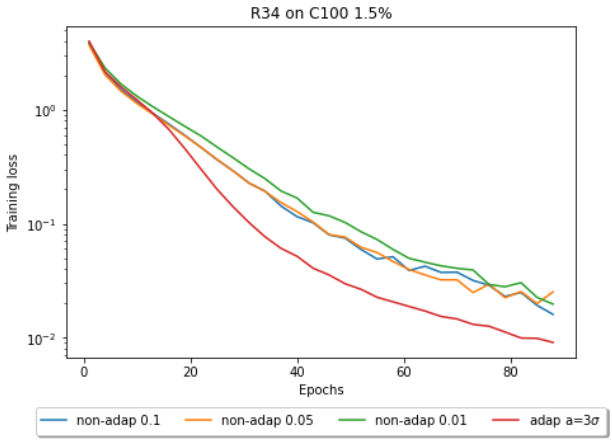}%
\caption{ResNet34,CIFAR100,1.5\%}%
\label{r34c100p1}%
\end{subfigure}\hfill%
\begin{subfigure}{.4\columnwidth}
\includegraphics[width=\columnwidth]{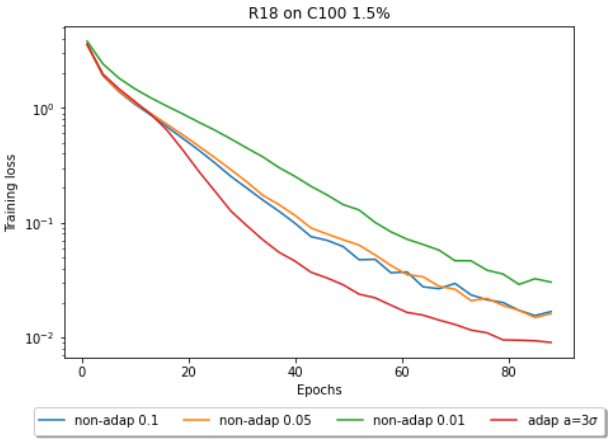}%
\caption{ResNet18,CIFAR100,1.5\%}%
\label{r10c10p1}%
\end{subfigure}\hfill%
\begin{subfigure}{.4\columnwidth}
\includegraphics[width=\columnwidth]{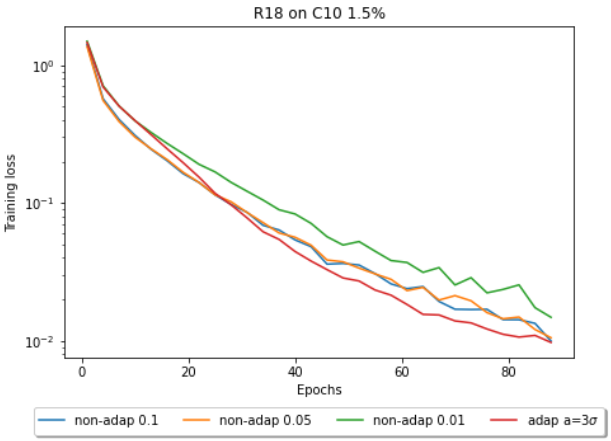}%
\caption{ResNet18,CIFAR10,1.5\%}%
\label{subfigc}%
\end{subfigure}%

\caption{Training loss at $\approx \ 1\%$ compression}
\label{fig:1}
\end{figure*}

\begin{figure*}%
\centering
\begin{subfigure}{.4\columnwidth}
\includegraphics[width=\columnwidth]{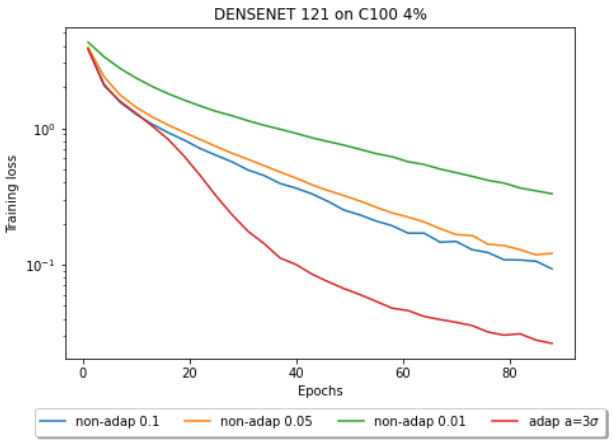}%
\caption{DenseNet121,CIFAR100, 4\% compression}%
\label{r34c10p10}%
\end{subfigure}\hfill%
\begin{subfigure}{.4\columnwidth}
\includegraphics[width=\columnwidth]{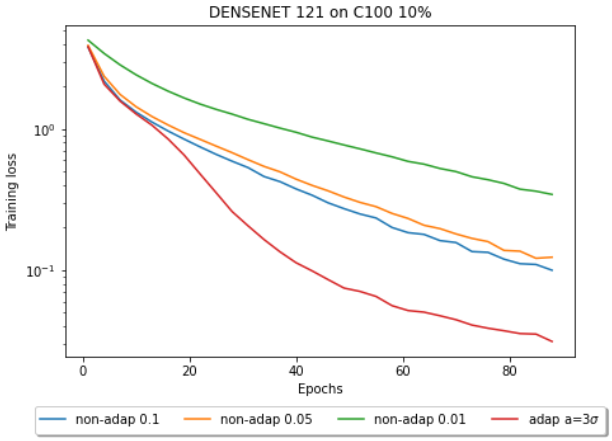}%
\caption{DenseNet121,CIFAR100, 10\% compression}%
\label{r34c100p10}%
\end{subfigure}\hfill%
\caption{Training loss of DenseNet 121 on CIFAR100,  at $4\%$, $10\%$ compression }
\label{fig:2}
\begin{subfigure}{.4\columnwidth}
\includegraphics[width=\columnwidth]{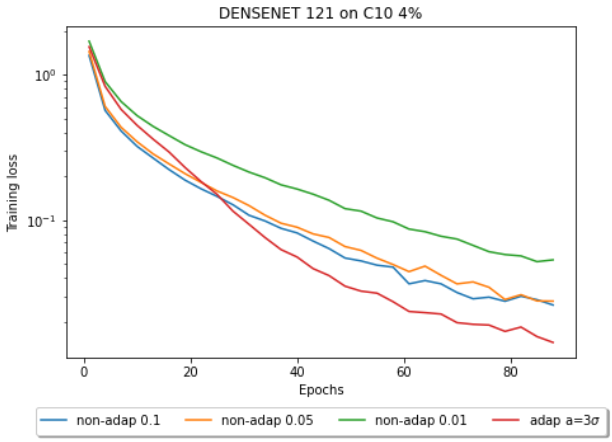}%
\caption{DenseNet121,CIFAR10, 4\% compression}%
\label{r18c100p10}%
\end{subfigure}\hfill%
\begin{subfigure}{.4\columnwidth}
\includegraphics[width=\columnwidth]{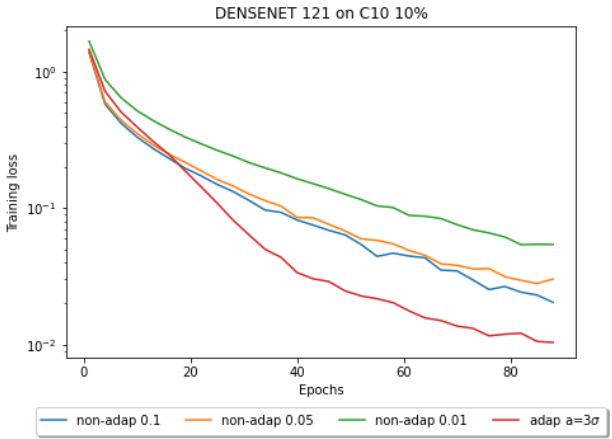}%
\caption{DenseNet121,CIFAR10, 10\% compression}%
\label{r18c10p10}%
\end{subfigure}%
\caption{Training loss of DenseNet 121 on CIFAR10,  at $4\%$, $10\%$ compression}
\label{fig:3}
\end{figure*}

In this section, we describe our experimental setup, discuss the deep neural network architectures we use and present simulation results to demonstrate the performance of our proposed method. Our simulations have been performed on NVIDIA GTX 1080 GPUs. 

\subsection{Experimental Setup}\label{subsec:exes}
We test our CSGD-ASSS algorithm on neural network architectures ResNet 18, ResNet 34 and DenseNet 121, on CIFAR 100 and CIFAR 10 datasets. Consistent with compression schemes in previous works \cite{basu_qsparselocalsgd,karimireddy2019error}, the compression is done at each layer. As in \cite{basu_qsparselocalsgd}, the layers with less than $1000$ parameters are not compressed.  We use a batch size of $64$ in our experiments and train for $90$ epochs. The training error is typically of the order of $10^{-2}$.

We set the initial value of $\alpha_{\mathrm{max}} =0.1$. In each subsequent iteration, the value of $\alpha_{\mathrm{max}}$ is updated as $ \alpha_{\mathrm{max}} = \omega \alpha_{t-1}$, where $\alpha_{t-1}$ is the step-size returned by the Armijo step-size search~(Step~\ref{step:ls_csgd}) in CSGD-ASSS in the previous iteration. In our simulations, $\omega = 1.2$ and the step-size decay factor used in the Armijo-step search, $\rho = 0.8$. 
This renders the algorithm computationally efficient for neural network training tasks, as discussed in Section~\ref{subsec:sim_res}.

For the Armijo step-size search parameter $\sigma$, we use the value $0.1$ as in non-compressed SGD with Armijo step-size search\cite{stoc_painless}. Motivated by the analysis of scaled GD presented in  Section~\ref{sec:det_gd} in the Appendix, we set the scaling factor $a$ to be a multiple of $\sigma$. In our simulations, we set $a = 3\sigma$ for all neural network training tasks.

In our experiments we choose arbitrary compression rates for the neural networks. We train ResNet 34, ResNet 18 on CIFAR 100 and CIFAR 10 datasets at compression~($\gamma$ values) $ 1.5\%$ and $10\%$, and DenseNet 121 on CIFAR 100 and CIFAR 10 at $4\%$ and $10\%$ compression. This means that, for example, in ResNet 34, CSGD-ASSS uses only $1.5\%$ and $10\%$ of gradient components in each layer. We compare our CSGD-ASSS algorithm with the non-adaptive compressed SGD of \cite{aji_hea} on typically used step-sizes $0.1,0.05,0.01$. The reason for using the algorithm in \cite{aji_hea} as a baseline for comparison is that prevalent \emph{non-adaptive} compressed SGD algorithms build on the $top_k$ gradient compression operator with memory feedback introduced in \cite{aji_hea}. 

\subsection{Simulation Results}\label{subsec:sim_res}

From Figures \ref{fig:1},\ref{fig:2},\ref{fig:3} , we observe that our proposed CSGD-ASSS algorithm (denoted by $3\sigma$) outperforms fixed/non-adaptive step-size $top_k$ compressed SGD with step-sizes $0.1,0.05,0.01$ (denoted by non-adap $0.1$, non-adap $0.05$, non-adap $0.01$, respectively) on ResNet 34 and ResNet 18, on both the CIFAR 100 and CIFAR 10 datasets in training at $ 1\%,$ compression. The CSGD-ASSS algorithm also has the best performance on DenseNet 121 on the CIFAR datasets at $4\%$ and $10\%$ compression. In the Appendix, we show in Figures~\ref{fig:4},\ref{fig:5} and Table~\ref{tab:val} that the CSGD-ASSS algorithm outperforms non-adaptive compressed SGD at $10\%$ compression on ResNet 18, ResNet 34 on CIFAR 100 and CIFAR 10 datasets, and that the validation accuracy is also competitive.


\noindent \textbf{\textit{Note on computational complexity:}} The Armijo step-size search of the CSGD-ASSS algorithm with $\omega=1.2$ and $\rho=0.8$, on ResNets and DenseNets, on average computes less than one additional forward pass compared to non-adaptive compressed SGD. This means that each iteration of Step~\ref{step:ls_csgd} of  CSGD-ASSS (Algorithm~\ref{alg:biased_compressed_sgd}) involves the computation of the gradient once, and subsequently multiplying the gradient with step-size $\alpha_t$ and evaluating the stopping condition~(Step \ref{step:armi_stop}) of the Armijo step-size search~(Algorithm~\ref{alg:armijo_ls}) twice. Evaluating the gradient is  computationally several orders more expensive compared to a forward pass in the SGD algorithm; in the federated learning~\cite{kairouz2021advances} setup, the forward pass steps can be parallelized in the worker nodes. With each gradient being several GigaBytes in size, and several hundred worker nodes communicating with the master node, the communication time would still remain the major bottleneck when implementing our proposed algorithm in federated/distributed learning scenarios.  

\subsection{The Necessity of Scaling for Convergence of CSGD-ASSS }\label{subsec:scaling_fundamental_plots}
 In this section, we illustrate through examples that scaling is necessary for convergence of compressed SGD with Armijo step-size search. To this end, we consider interpolated linear regression in the convex setting and neural-networks in the non-convex setting as examples. 
 
In interpolated linear regression, the loss function is 
\begin{equation}
\begin{aligned}
f(x) = \frac{1}{n} \sum_{i=1}^{n} (\langle a_i,x \rangle -b_i)^2.
\end{aligned}    
\end{equation}
where $\exists\ x^*$ such that $\langle a_i,x^* \rangle = b_i$, $\forall\ i\in[n]$. In our experiments, we choose $n=10000$, $\{a_i,x\} \in \mathbb{R}^{1024}$ and use $top_k$ compression with compression ratio~$\frac{k}{d} = 1\%$. The elements of $x^*$ are generated from $\mathcal{N}(0,1)$ and the elements of $a_i$ are distributed as $\mathcal{N}(0,1)$ and $\mathcal{N}(0,10)$ in Figures~\ref{fig:ilr_sig_1} and \ref{fig:ilr_sig_10}, respectively. The loss increases exponentially without scaling even in the convex setting. This clearly demonstrates that for compressed SGD with Armijo step-size search, scaling is a fundamental requirement and not just a proof technicality.
In Figure~\ref{fig:scale_nn}, the CSGD-ASSS algorithm is evaluated on ResNet-34 and ResNet-18 on CIFAR-100 dataset with scaling $a=3\sigma$ and $a=1$~(no scaling) at $\approx 1\%$ compression. The plot demonstrates that the conclusions regarding the convergence of the CSGD-ASSS algorithm also hold in (non-convex) neural network training tasks. 
\begin{figure*}%
\centering
\begin{subfigure}{.32\columnwidth}
\includegraphics[width=\columnwidth]{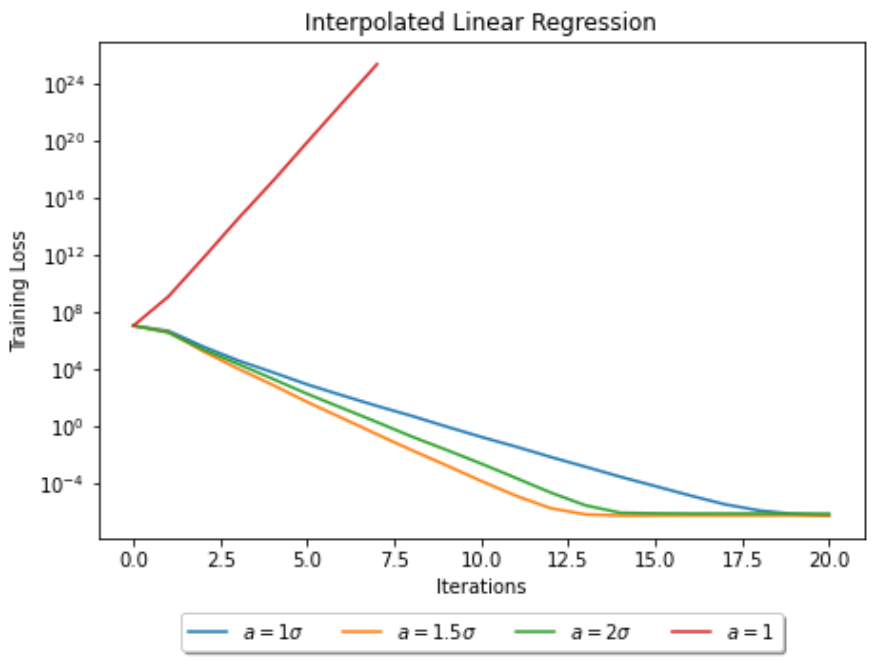}%
\caption{Interpolated linear regression \\* {\color{white}vvvvvvv} $a_i \in \mathcal{N}(0,1)$ }%
\label{fig:ilr_sig_1}%
\end{subfigure}\hfill%
\begin{subfigure}{.32\columnwidth}
\includegraphics[width=\columnwidth]{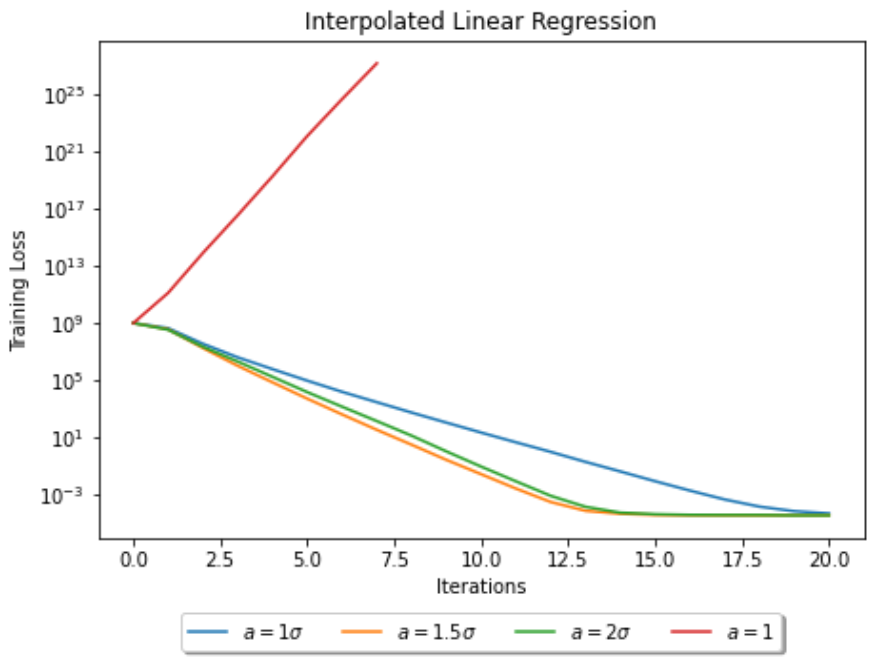}%
\caption{{Interpolated linear regression \\ {\color{white}vvvvvvv} $a_i \in \mathcal{N}(0,10)$}}%
\label{fig:ilr_sig_10}%
\end{subfigure}\hfill%
\begin{subfigure}{.32\columnwidth}
\includegraphics[width=\columnwidth]{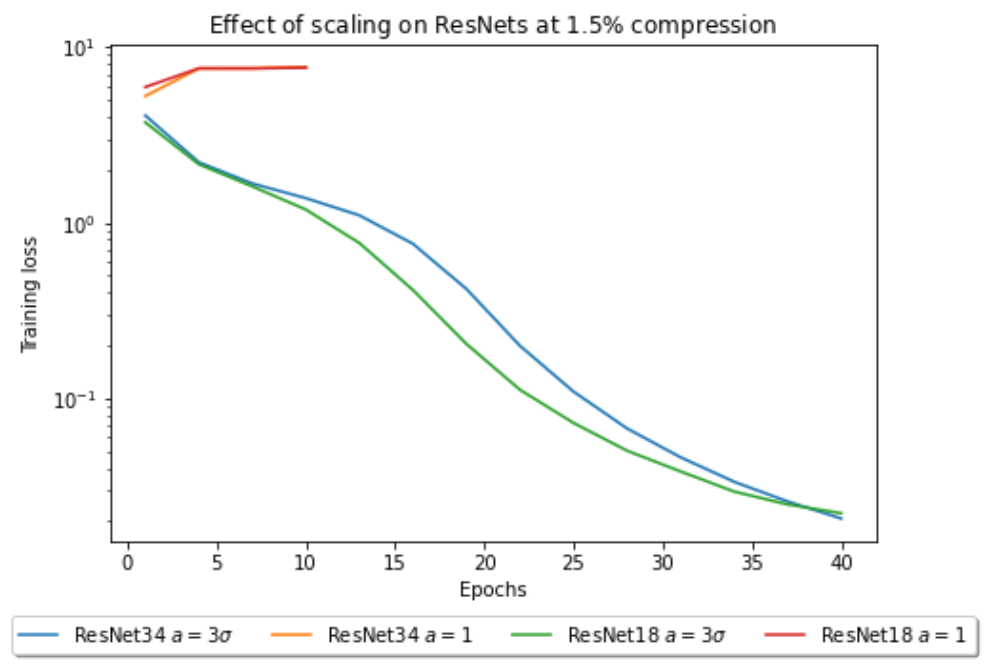}%
\caption{ResNets,CIFAR100,1.5\% comp- {\color{white}vvvvvvv}-ression}%
\label{fig:scale_nn}%
\end{subfigure}\hfill%

\caption{Scaled vs non-scaled CSGD-ASSS at $\approx \ 1\%$ compression}
\label{fig:scaling}
\end{figure*}

\section{Conclusion}
We have motivated and presented a scaling technique for Armijo step-size search based adaptive gradient descent methods. We showed how the scaling technique can be used in establishing convergence results, in convex, strongly convex, and non-convex settings, when Armijo step-size search is applied in the context of compressed SGD. In simulations, we have shown that our proposed Compressed SGD with Armijo Step-Size Search and Scaling (CSGD-ASSS) algorithm outperforms the corresponding non-adaptive compressed SGD algorithm on ResNet 18, ResNet 34 and DenseNet architectures trained on CIFAR 100 and CIFAR 10 datasets at various levels of compression. We have also demonstrated through simulations the important role of scaling for convergence of  compressed SGD with Armijo step-size search. A distributed version of the CSGD-ASSS algorithm is presented in the Appendix with convergence guarantees. For future work, one could study the use of  modifications such as local iterations, momentum, other error-feedback operators, and adaptive gradient methods with the CSGD-ASSS algorithm.

\newpage

\bibliographystyle{IEEEtran}
\bibliography{collab}


\newpage
\hrule
\begin{center}
{ \huge Appendix}
\end{center}
\hrule

\section{Mathematical Preliminaries}

We begin by stating some useful results in Lemmas~\ref{lem:math_prelim_tri} and \ref{lem:math_prelim_tri2}, whose proofs are straightforward. 

\begin{lemma}\label{lem:math_prelim_tri}
For  vectors $v_1,v_2 \in \mathbb{R}^n$, and  $p>0$,
\begin{equation}
    2\langle v_1,v_2 \rangle \leq p \norm{v_1}^2 + \frac{1}{p} \norm{v_2}^2 
\end{equation}
\end{lemma}

\begin{lemma}\label{lem:math_prelim_tri2}
For  vectors $v_1,v_2 \in \mathbb{R}^n$, and $r>0$,
\begin{equation}
    \norm{v_1 + v_2}^2  \leq (1+r) \norm{v_1}^2 + (1+\frac{1}{r} )\norm{v_2}^2 
\end{equation} 
\end{lemma}

\begin{lemma}\label{lem:math_prelim_error}
In CSGD-ASSS algorithm, $m_t = x_t - \hat{x}_{t}$ 
\end{lemma}

\begin{proof}
See \cite{stich2018sparsified}.
\end{proof}

\begin{lemma} \label{lem:math_prelim_compress}
For $top_k$ compression operator with $\gamma = \frac{k}{d}$, and a vector $v_1 \in \mathbb{R}^n$,
\begin{equation}
\norm{v_1 - top_k(v_1)}^2 \leq (1-\gamma)||v_1||^2
\end{equation}

\end{lemma}

\begin{proof}
See \cite{basu_qsparselocalsgd}.
\end{proof}

\begin{lemma}\label{lem:math_prelim_armijo_stop}
The Armijo step-size search stopping condition for a data-point $i_t$ sampled at time $t$ 
\begin{equation}
f_{i_t}(x_t - \alpha_t \nabla f_{i_t}(x_t)) - f_{i_t}(x_t) \leq -\sigma \alpha_t ||\nabla f_{i_t}(x_t)||^2    
\end{equation}
\end{lemma}
is satisfied by all $\alpha_t \in [0, \frac{2(1-\sigma)}{L_{\mathrm{max}}}]$ where $L_{\mathrm{max}} = \max_{i\in [n]}{\{L_i\}}$

\begin{proof}
 See \cite{armijo_ls},\cite{stoc_painless}.
\end{proof}
\begin{lemma}\label{lem:armijo_bound}
The Armijo step-size search rule returns a step-size $\alpha_t\in [\tilde{\alpha}_{\mathrm{min}}\rho,\alpha_{\mathrm{max}}]$  where $\tilde{\alpha}_{\mathrm{min}} \triangleq \frac{2(1-\sigma)}{L_{\mathrm{max}}}$
\end{lemma}
\begin{proof}
See \cite{armijo_ls}
\end{proof}

\begin{remark}
In our proofs, for the Armijo step-size search $\alpha_{\mathrm{min}}= \tilde{\alpha}_{\mathrm{min}}\rho$ and $\alpha_{\mathrm{min}} \leq \alpha_t \leq \alpha_{\mathrm{max}}$. The step-size $\eta_t$ for CSGD-ASSS is obtained by $\eta_t = a \alpha_t$ and hence $ \eta_t \in [\eta_{\mathrm{min}}, \eta_{\mathrm{max}}]$, where $\eta_{\mathrm{min}} = a \alpha_{\mathrm{min}} $ and 
$\eta_{\mathrm{max}} = a \alpha_{\mathrm{max}} $.
\end{remark}

\section{Proofs for the CSGD-ASSS Algorithm}

We first prove Theorem~\ref{thm:strong_convex}, since this proof is useful in the proofs of  Theorems~\ref{thm:bcmf_convex},~\ref{thm:non_conv_asss}.

\subsection{Proof of Theorem~\ref{thm:strong_convex}}

\begin{proof}

From perturbed iterate analysis, $\hat{x}_{t+1} \triangleq \hat{x}_t - \eta_t \nabla f_{i_t}(x_t)$ where $i_t$ is the data-point sampled at time $t$ and $\hat{x}_0 = x_0 $. Hence,
\begin{equation}
\begin{aligned}
|| \hxtt - x^*||^2 &= ||\hxt - x^* - \eta_t \nabla f_{i_t}(x_t)  ||^2\\    
 &= ||\hxt - x^*||^2 + \eta_t^2 ||\nabla f_{i_t}(x_t)||^2 -2\langle \hxt-x^*,\eta_t \nabla f_{i_t}(x_t) \rangle\\
 &= ||\hxt - x^*||^2 + \eta_t^2 ||\nabla f_{i_t}(x_t)||^2 -2 \langle \hxt - x_t + x_t - x^* , \eta_t \nabla f_{i_t}(x_t)\rangle\\
 &= ||\hxt -x^*||^2 + \eta_t^2 ||\nabla f_{i_t}(x_t)||^2 - 2\langle \hxt-x_t , \eta_t \nabla f_{i_t}(x_t) \rangle -2 \langle x_t-x^*, \eta_t \nabla f_{i_t}(x_t)  \rangle\\
 &= ||\hxt - x^*||^2 + \eta_t^2 ||\nabla f_{i_t}(x_t)||^2 - 2\langle \hxt-x_t , \eta_t \nabla f_{i_t}(x_t) \rangle + 2 \langle x^* - x_t, \eta_t \nabla f_{i_t}(x_t)  \rangle.
 \end{aligned}
\end{equation}

If $f_{i_t}$ is $\mu_{i_t}$-strongly convex,

\begin{equation}\label{eq:th7_strongconv}
\langle x^* - x_t,  \nabla f_{i_t}(x_t)  \rangle \leq f_{i_t}(x^*) - f_{i_t}(x_t) - \frac{\mu_{i_t}}{2} ||x^* - x_t||^2. 
\end{equation}

If $f_{i_t}$ is convex but not strongly convex, then $\mu_{i_t} =0$. In the statement of the theorem, the only assumption is that $\exists\ i\in[n]$ such that $\mu_{i}>0$. 


From triangle inequality, 
\begin{equation}\label{eq:th7_triangle}
-||x^* - x_t ||^2 \leq ||x_t - \hxt||^2  - \frac{1}{2} ||\hxt - x^*||^2.
\end{equation}

Substituting \eqref{eq:th7_triangle} in \eqref{eq:th7_strongconv} and  $m_t = x_t - \hat{x}_t$~(by Lemma~\ref{lem:math_prelim_error}) and defining $e_t \triangleq  f_{i_t}(x_t) - f_{i_t}(x^*)$,
\begin{equation}
\begin{aligned}
2 \langle x^* - x_t , \eta_t \nabla f_{i_t}(x_t) \rangle &\leq 2 \eta_t (f_{i_t}(x^*) - f_{i_t}(x_t)) + \mu_{i_t} \eta_t \left(||x_t - \hxt ||^2 -\frac{1}{2} ||\hxt - x^*||^2 \right) \\
&= 2 \eta_t (-e_t) + \mu_{i_t} \eta_t ||m_t||^2 - \frac{\mu_{i_t}}{2} \eta_t ||\hxt - x^*||^2.
\end{aligned}
\end{equation}
Hence,
\begin{equation}\label{eq:49}
\begin{aligned}
||\hxtt - x^* ||^2 &\leq ||\hxt - x^* ||^2 + \eta_t^2 ||\nabla f_{i_t}(x_t)||^2  - 2 \langle \hxt - x_t , \eta_t \nabla f_{i_t}(x_t) \rangle -2 \eta_t e_t + \mu_{i_t} \eta_t ||m_t||^2 \\
& \hspace{2em}-\frac{\mu_{i_t} \eta_t }{2} ||\hxt - x^*||^2  \\
&= (1- \frac{\mu_{i_t}\eta_t}{2}) ||\hxt - x^*||^2 + \eta_t^2 ||\nabla f_{i_t}(x_t)||^2 - 2\eta_t e_t + \mu_{i_t} \eta_t ||m_t||^2 \\
& \hspace{2em}+ 2\langle  x_t -\hxt  , \eta_t \nabla f_{i_t}(x_t) \rangle.
\end{aligned}    
\end{equation}

From Lemma~\ref{lem:math_prelim_tri}, 
\begin{equation}
    2 \eta_t \langle x_t - \hxt, \nabla f_{i_t}(x_t) \rangle \leq q_t \eta_t ||x_t - \hxt ||^2 + \frac{\eta_t}{q_t} ||\nabla f_{i_t}(x_t)||^2,   \;\;\forall q_t> 0 .
\end{equation}
Substituting in~\eqref{eq:49},

\begin{equation}\label{eq:scf1}
\begin{aligned}
||\hxtt - x^* ||^2 &\leq \left(1- \frac{\mu_{i_t}\eta_t}{2}\right) ||\hxt - x^*||^2 + \eta_t^2 ||\nabla f_{i_t}(x_t)||^2 - 2\eta_t e_t + \mu_{i_t} \eta_t ||m_t||^2 \\
& \hspace{2em}+ q_t \eta_t||x_t - \hxt ||^2 + \frac{\eta_t}{q_t}||\nabla f_{i_t}(x_t)||^2.
\end{aligned}    
\end{equation}

From the compression property in Lemma~\ref{lem:math_prelim_compress} and the memory update step $m_{t+1} = (m_t + \eta_t \nabla f_{i_t}(x_t)) - top_k(m_t + \eta_t \nabla f_{i_t}(x_t))$ in the CSGD-ASSS algorithm,

\begin{equation}\label{eq:scf21}
\begin{aligned}
||m_{t+1}||^2 &= \norm{(m_t + \eta_t \nabla f_{i_t}(x_t)) - top_k(m_t + \eta_t \nabla f_{i_t}(x_t))}^2 \\
&\leq (1-\gamma) ||m_t + \eta_t \nabla f_{i_t}(x_t) ||^2. \\
\end{aligned}    
\end{equation}

From Lemma~\ref{lem:math_prelim_tri2},
\begin{equation}
 \norm{m_{t} + \eta_t \nabla f_{i_t}(x_t) }^2 \leq (1+r)\norm{m_t}^2 + (1+\frac{1}{r}) \eta_t^2 \norm{\nabla f_{i_t}(x_t)}^2, \ \forall\ \   r > 0.   
\end{equation}
 Substituting this in~\eqref{eq:scf21}, 

\begin{equation}\label{eq:scf2}
\begin{aligned}
||m_{t+1}||^2 
 & \leq (1-\gamma)(1+r) ||m_t||^2 + (1-\gamma) (1+ \frac{1}{r}) \eta_t^2 ||\nabla f_{i_t}(x_t)||^2
\end{aligned}    
\end{equation}

for some $r>0$. Adding \eqref{eq:scf1} and \eqref{eq:scf2}, we obtain

\begin{equation}\label{eq:scf3}
\begin{aligned}
||\hxtt - x^* ||^2 + ||m_{t+1}||^2  &\leq (1- \frac{\mu_{i_t}\eta_t}{2}) ||\hxt - x^*||^2 + \eta_t^2 ||\nabla f_{i_t}(x_t)||^2 - 2\eta_t e_t \\
&+ \mu_{i_t} \eta_t ||m_t||^2 + q_t \eta_t ||x_t - \hxt ||^2 + \frac{1}{q_t} \eta_t ||\nabla f_{i_t}(x_t)||^2 \\ 
& +(1-\gamma)(1+r) ||m_t||^2 \\
&+(1-\gamma) (1+ \frac{1}{r}) \eta_t^2 ||\nabla f_{i_t}(x_t)||^2.
\end{aligned}    
\end{equation}

Setting $p \triangleq q_t a \alpha_t$ where $p \in \mathbb{R}^+$~($p$ can be set independent of $t$ since $q_t$ is arbitrary),

\begin{equation}\label{eq:scf5}
\begin{aligned}
||\hxtt - x^* ||^2 + ||m_{t+1}||^2  &\leq (1- \frac{\mu_{i_t}\eta_t}{2}) ||\hxt - x^*||^2 + \eta_t^2 ||\nabla f_{i_t}(x_t)||^2 - 2\eta_t e_t + \mu_{i_t} \eta_t ||m_t||^2 \\
&+\frac{p}{a \alpha_t} \eta_t ||x_t - \hxt ||^2+ \frac{a \alpha_t}{p} \eta_t ||\nabla f_{i_t}(x_t)||^2 \\ &+(1-\gamma)(1+r) ||m_t||^2 \\
&+ (1-\gamma) (1+ \frac{1}{r}) \eta_t^2 ||\nabla f_{i_t}(x_t)||^2.
\end{aligned}    
\end{equation}

Consider the Armijo step-size search stopping condition

\begin{equation}
f_{i_t}(\txtt) - f_{i_t}(x_t) \leq - \sigma \alpha_t ||\nabla f_{i_t}(x_t)||^2.    
\end{equation}

This implies,
\begin{equation}\label{eq:scf4}
\begin{aligned}
 \eta_t^2 ||\nabla f_{i_t}(x_t) ||^2 \overset{(\mathbf{\phi})}{\leq} \frac{a^2 \alpha_t}{\sigma} (f_{i_t}(x_t) - f_{i_t}(x^*)).
\end{aligned}    
\end{equation}
Note that $(\phi)$ holds due to the interpolation condition.
Substituting \eqref{eq:scf4} in \eqref{eq:scf5}, 
\begin{equation}\label{eq:scf8}
\begin{aligned}
||\hxtt - x^* ||^2 + ||m_{t+1}||^2  &\leq (1- \frac{\mu_{i_t}\eta_t}{2}) ||\hxt - x^*||^2 + \Big(\mu_{i_t} \eta_t  + p  +(1-\gamma)(1+r) \Big)||m_t||^2 \\ 
& - \alpha_t e_t \Big(2 a  - \frac{a^2 }{\sigma} - \frac{a^2 }{\sigma p}   - (1-\gamma)(1+ \frac{1}{r}) \frac{a^2 }{\sigma}  \Big).
\end{aligned}    
\end{equation} 
Let $\mu_{\mathrm{max}} \triangleq \max_{i\in[n]} \mu_i$. Using $\eta_t = a \alpha_t $ and $\alpha_t \leq \alpha_{\mathrm{max}}$, where $\alpha_{\mathrm{max}}$ is the starting value of the Armijo step-size search algorithm, we get 
\begin{equation}
\begin{aligned}
\Big(\mu_{i_t} \eta_t  + p  +(1-\gamma)(1+r) \Big) \leq \Big(\mu_{\mathrm{max}} a \alpha_{\mathrm{max}}  + p  +(1-\gamma)(1+r) \Big).
\end{aligned}    
\end{equation}


Notice that in \eqref{eq:scf8}, $p,r$ can be chosen arbitrarily. Define $\beta_1(p,r), \tilde{a}_1(p,r)$ and $\tilde{a}_2(p,r)$ as 

 \begin{equation}\label{eq:scf7}
 \begin{aligned}
 \beta_1(p,r) \triangleq \Big(\mu_{\mathrm{max}} a \alpha_{\mathrm{max}}  + p  +(1-\gamma)(1+r) \Big) \\
 \tilde{a}_1(p,r) \triangleq 2/ \Big( \frac{1}{\sigma} + \frac{1}{\sigma p} + \frac{(1-\gamma)(1+ \frac{1}{r})}{\sigma} \Big) \\
 \tilde{a}_2(p,r) \triangleq \Big(2 a  - \frac{a^2 }{\sigma} - \frac{a^2 }{\sigma p}   - (1-\gamma)(1+ \frac{1}{r}) \frac{a^2 }{\sigma}  \Big)
 \end{aligned}     
 \end{equation}

 If $a < \tilde{a}_1(p,r)$, then $\tilde{a}_2(p,r) >0$. If $p,r,a$ are chosen such that $\beta_1(p,r) <1 $ and $a < \tilde{a}_1(p,r)$, \eqref{eq:scf8} simplifies to  

\begin{equation}\label{eq:scf9}
\begin{aligned}
||\hxtt - x^* ||^2 + ||m_{t+1}||^2  &\leq (1- \frac{\mu_{i_t}a \alpha_t}{2}) ||\hxt - x^*||^2 + \beta_1(p,r) ||m_t||^2.
\end{aligned}    
\end{equation}

The existence of such $p,r,a$ is proven in Lemma~\ref{lem:ex_bound_lemma}. 
The step-size in the Armijo step-size search satisfies $\tilde{\alpha}_{\mathrm{min}}\rho \leq \alpha_t $. Using this property and taking expectation with respect to the data point $i_t$ conditioned on the entire past therein, 

\begin{equation}\label{eq:scf10}
\begin{aligned}
E_{i_t}[||\hxtt - x^* ||^2 + ||m_{t+1}||^2]  &\leq (1- \frac{E_{i_t}[\mu_{i_t}]a \alpha_{\mathrm{min}}\rho}{2}) ||\hxt - x^*||^2 + \beta_1(p,r) ||m_t||^2.
\end{aligned}    
\end{equation}
By assumption, $\exists \ i\in[n]$ such that $\mu_{i} >0$. Hence, $$\Bar{\mu} = E_{i_t}[\mu_i] =  \frac{\sum_{i=1}^{n} \mu_i}{n} >0.$$

This implies that

$$ 0< (\Bar{\mu}a \tilde{\alpha}_{\mathrm{min}} \rho)/2 = \Bar{\mu} a \frac{1}{L_{\mathrm{max}}} (1-\sigma) \rho <1,$$
since $\Bar{\mu}\leq L_{\mathrm{max}}$, the values $\sigma, \rho <1$ and $a<1$ by Lemma~\ref{lem:ex_bound_lemma}. Let 
$$\beta_2 \triangleq (1- \frac{\Bar{\mu}a \tilde{\alpha}_{\mathrm{min}}\rho}{2}).$$

It follows that $0<\beta_2 <1$. Let $\hat{\beta} \triangleq \max\{\beta_1(p,r),\beta_2\}$. Since $\beta_1(p,r)<1 $, $ \beta_2 < 1$, this implies  $\hat{\beta}<1$. Using this in \eqref{eq:scf10}, 

\begin{equation}\label{eq:scf11}
\begin{aligned}
E_{i_t}[||\hxtt - x^* ||^2 + ||m_{t+1}||^2]  &\leq  \hat{\beta}(||\hxt - x^*||^2 +  ||m_t||^2 ).
\end{aligned}    
\end{equation}

Taking expectation with respect to the entire process,
\begin{equation}\label{eq:scf12}
\begin{aligned}
E [||\hxtt - x^* ||^2 + ||m_{t+1}||^2]  &\leq  \hat{\beta} E[(||\hxt - x^*||^2 +  ||m_t||^2 )].
\end{aligned}    
\end{equation}
Hence,
\begin{equation}\label{eq:scf13}
\begin{aligned}
E [||\hxt - x^* ||^2 + ||m_{t}||^2]  &\leq  (\hat{\beta})^t E[(||\hat{x}_0 - x^*||^2 +  ||m_0||^2 )].
\end{aligned}    
\end{equation}
Since $\|m_{t}\| = \|x_t - \hxt\|$, 
\begin{equation}
||x_t -x^*||^2 \leq 2 (||x_t - \hxt||^2) + 2 (|| \hxt - x^* ||^2) = 2 (||m_t||^2 + || \hxt - x^* ||^2 ) .    
\end{equation}

At $t=0$, $m_0 =0$ and hence,
\begin{equation}\label{eq:scf14}
\begin{aligned}
E[||x_t - x^*||^2] & \leq 2 (E [||\hxt - x^* ||^2 + ||m_{t}||^2] )\\
&\leq  2 (\hat{\beta})^t E[(||\hat{x}_0 - x^*||^2 )].
\end{aligned}    
\end{equation}

In the perturbed iterate analysis, $\hat{x}_0 = x_0$. Substituting this we get,

\begin{equation}
\begin{aligned}
E[||x_t - x^*||^2]
\leq  2 (\hat{\beta})^t E[(||{x}_0 - x^*||^2 )].
\end{aligned}    
\end{equation}

Hence, we obtain the geometric convergence stated in the theorem.
\end{proof}

To show the existence of $p,r,a$ such that $\beta_1(p,r) <1 $ and $a < \tilde{a}_1(p,r)$ we prove Lemma~\ref{lem:ex_bound_lemma}, which is based on Lemmas~\ref{lem:conv_problem} and \ref{lem:mu_bound} proved below.

\begin{lemma}\label{lem:conv_problem}
Consider the convex optimization problem with variables $s,z,\psi \in \mathbb{R}$ and $0<\psi<1$.

\begin{equation}\label{conv_problem}
\begin{aligned}
    \min_{s,z} \ &g(s,z) = \frac{1}{s} + \psi \frac{1}{z}\\
    \mathrm{s.t.} \ &(s+ \psi (1+z)) \leq 1\\
    &s \geq 0 ,z\geq 0
\end{aligned}
\end{equation}

The function $g(s,z)$ attains the minimum value $ \frac{(1+\psi)^2}{(1-\psi)}$ at $(s^*,z^*) = ( \frac{1-\psi}{1+\psi},\frac{1-\psi}{1+\psi})$.

\end{lemma}

\begin{proof}

To solve the convex optimization problem, we consider the dual of the problem and apply the approach of Proposition~{3.3.1} in~\cite{bertsekas_book}. 

The constrains $s\geq 0$ and $z\geq 0$ are inactive since $s=0$ or $z=0$ implies the objective $g(s,z)=\infty$. Since only $(s+ \psi (1+z)) \leq 1 $ can be active, all local minimum are regular. 

Consider the Lagrangian
\begin{equation}
L(s,z,\lambda_1,\lambda_2,\lambda_3) = g(s,z) + \lambda_1 (s+ \psi (1+z) -1)  + \lambda_2 (-s) + \lambda_3 (-z)    
\end{equation}

The values $s^*,z^*$ are a local minimum of the convex problem~\eqref{conv_problem} if $\exists \lambda_1^*,\lambda_2^*,\lambda_3^*$ such that $\lambda_2^* =0,\lambda_3^* =0$ and $\nabla_{s,z} L(s^*,z^*,\lambda_1^*,\lambda_2^*,\lambda_3^*) =0$.

Solving for $L(s^*,z^*,\lambda_1^*,\lambda_2^*,\lambda_3^*) =0$ we get $s^*=z^*$ and $\lambda_1^* = \frac{1}{s^*}$ at any local minimum. Hence $\lambda_1^* >0$ which implies that the constraint $(s^*+ \psi (1+z^*)) \leq 1$ is active or tight. 

Substituting $s^*=z^*$ in the active constraint,

\begin{equation}
    s^* + \psi (1+s^*) = 1.
\end{equation}

This implies 

\begin{equation}
    s^* = z^* = \frac{1-\psi}{1+ \psi}.
\end{equation}
Hence there is a unique local minimum. The minimum value of the objective $g(s,z)$ subject to the constraints is 

\begin{equation}
\begin{aligned}
g(s^*,z^*) &= \frac{1}{s^*} + \psi \frac{1}{z^*} \\
&= \frac{1}{s^*}(1+\psi)\\
&= \frac{(1+\psi)^2}{1-\psi}. 
\end{aligned}
\end{equation}

\end{proof}

\begin{lemma}\label{lem:mu_bound}
 Let $f_i(x)$, $i\in[n]$ be a set of $\mu_i $ strong-convex functions s.t $\mu_i >0$ for some $i>0$ and $\mu_{\mathrm{max}}\triangleq \max_{i\in[n]} \mu_i$. Then $\mu_{\mathrm{max}} \leq \xi $, $\forall\ \xi>0$ is justified. 
\end{lemma}

\begin{proof}
The case where $\xi\geq\mu_{\mathrm{max}}$ is trivial. Hence consider the case where $\xi < \mu_{\mathrm{max}}$. Without loss of generality, let $ \xi < \mu_1$ where $\mu_1$ is the strong convexity constant of $f_1$. This also implies that 

\begin{equation}
\begin{aligned}
\nabla^2 f_1(x) \succeq \mu_1 I.     
\end{aligned}
\end{equation}

Since $\xi < \mu_1$, it also implies that 

\begin{equation}
\begin{aligned}
\nabla^2 f_1(x) \succeq \xi I.      
\end{aligned}
\end{equation}

Hence $f_1$ is $\xi $ strongly convex. It is thus justified to assume $\mu_1 = \xi$ and set $\mu_{\mathrm{max}} \leq \xi$.
\end{proof}

We show the following lemma based on Lemmas~\ref{lem:conv_problem} and \ref{lem:mu_bound}.

\begin{lemma}\label{lem:ex_bound_lemma}
 Let $\zeta \triangleq \frac{\sigma \gamma}{(2-\gamma)}$. For all $0<\epsilon < \zeta$ and $a \leq \zeta - \epsilon $, $\exists\ (p{(\epsilon)}, r{(\epsilon)})$ such that $\beta_1(p{(\epsilon)}, r{(\epsilon)})<1$ and $a<\tilde{a}_1(p{(\epsilon)}, r{(\epsilon)})$, where $p{(\epsilon)}, r{(\epsilon)}$ are functions of $\epsilon$. 
\end{lemma}

\begin{proof}
The functions $\beta_1(p,r)$ and $\tilde{a}_1(p,r)$ have been defined as
\begin{equation}
    \begin{aligned}
         \beta_1(p,r) &= \Big(\mu_{\mathrm{max}} a \alpha_{\mathrm{max}}  + p  +(1-\gamma)(1+r) \Big) \\
         \tilde{a}_1(p,r) &= 2/ \Big( \frac{1}{\sigma} + \frac{1}{\sigma p} + \frac{(1-\gamma)(1+ \frac{1}{r})}{\sigma} \Big).
    \end{aligned}
\end{equation}

Consider the convex problem 

\begin{equation}\label{eq:conv_sup_eq}
\begin{aligned}
    \sup_{p,r} \tilde{a}_1(p,r)  =  2 \sigma / \Big(1+ \frac{1}{p} + (1-\gamma)(1+ \frac{1}{r}) \Big)\\
    \mathrm{s.t.}\  (p + (1-\gamma)(1+ r)) < 1\\
    p \geq 0, r \geq 0
\end{aligned}    
\end{equation}

Equivalently, the problem can be stated as 
\begin{equation}\label{eq:conv_inf_eq}
\begin{aligned}
    \inf_{p,r}  \Big( \frac{1}{p} + (1-\gamma)(1+ \frac{1}{r}) \Big)\\
    \mathrm{s.t.}\  (p + (1-\gamma)(1+ r)) < 1 \\
    p \geq 0 , r \geq 0
\end{aligned}    
\end{equation}

Using Lemma~\ref{lem:conv_problem} with $\psi = (1-\gamma)$, the infimum is $\frac{(2-\gamma)^2}{\gamma} + (1-\gamma) $ and the infimum is attained at a point $(p,r)$ such that $(p + (1-\gamma)(1+r)) =1$. Hence, by substituting the infimum of \eqref{eq:conv_inf_eq} in \eqref{eq:conv_sup_eq},

\begin{equation}
\begin{aligned}
   \zeta \triangleq  \sup_{p,r} \tilde{a}_1(p,r) = \frac{\sigma \gamma}{(2-\gamma)},
\end{aligned}    
\end{equation}

subject to the constraints. By the continuity of the function $2 \sigma / \Big(1+ \frac{1}{p} + (1-\gamma)(1+ \frac{1}{r}) \Big) $ at $p >0, r >0$ and since $ (p + (1-\gamma)(1+ r)) < 1 ,p \geq 0, r \geq 0$ is a region in the first quadrant and $\zeta$ is attained at $(p^* ,r^*)$ such that $ (p^* + (1-\gamma)(1+ r^*)) = 1 $, for all $\epsilon>0$, $\exists~ (p{(\epsilon)}, r{(\epsilon)})$ and  $0<\delta<1 $ such that 

\begin{equation}
\begin{aligned}
   a_1(p{(\epsilon)}, r{(\epsilon)})> \zeta-\epsilon\\
   (p{(\epsilon)} + (1-\gamma)(1+ r{(\epsilon)})) = 1 - \delta
\end{aligned}    
\end{equation}

Let $\mu_{\mathrm{max}} \leq \frac{ (\delta - \tau)}{\alpha_{\mathrm{max}} \zeta }$ for some $0< \tau < \delta$. This assumption is justified by Lemma~\ref{lem:mu_bound}. Hence, 

\begin{equation}
\begin{aligned}
\beta_1(p{(\epsilon)}, r{(\epsilon)})  &\leq   \Big(\mu_{\mathrm{max}} \zeta \alpha_{\mathrm{max}}  + p{(\epsilon)}  +(1-\gamma)(1+r{(\epsilon)}) \Big)\\
&\leq (1-\tau) \\
& < 1.
\end{aligned}    
\end{equation}

To conclude, we have shown that for any $a \leq \zeta - \epsilon$, $\exists\ (p(\epsilon),r(\epsilon))$ such that $\tilde{a}_1(p(\epsilon),r(\epsilon))> \zeta-\epsilon$ and $\beta_1(p(\epsilon),r(\epsilon)) <1$. This implies that $\exists\ (p(\epsilon),r(\epsilon))$ s.t.

\begin{equation}
 \begin{aligned}
 a <\tilde{a}_1(p(\epsilon),r(\epsilon)) &&\\
 \beta_1(p(\epsilon),r(\epsilon))<1.
 \end{aligned}   
\end{equation}
\end{proof}

\begin{remark}
In the proof of Theorem~\ref{thm:strong_convex}, for all $a<\zeta$, $\exists$ $p,r$ such that $a < \tilde{a}_1(p,r) $ and $\beta_1(p,r)<1$ by Lemma~\ref{lem:ex_bound_lemma} and hence \eqref{eq:scf9} of Theorem~\ref{thm:strong_convex} holds. 
\end{remark}

\subsection{Proof of Theorem~\ref{thm:bcmf_convex}}

\begin{proof}
For the convex case, \eqref{eq:scf8} holds with $\mu_{i_t}=0$. Hence,

\begin{equation}\label{eq:cf1}
\begin{aligned}
||\hxtt - x^* ||^2 + ||m_{t+1}||^2  &\leq   ||\hxt - x^*||^2 + \Big(  p  +(1-\gamma)(1+r) \Big)||m_t||^2 \\ 
& - \alpha_t e_t \Big(2 a  - \frac{a^2 }{\sigma} - \frac{a^2 }{\sigma p}   - (1-\gamma)(1+ \frac{1}{r}) \frac{a^2 }{\sigma}  \Big).
\end{aligned}    
\end{equation} 

As in Theorem~\ref{thm:strong_convex}, we define 
 \begin{equation}\label{eq:scf7}
 \begin{aligned}
 \tilde{a}_1(p,r) = 2/ \Big( \frac{1}{\sigma} + \frac{1}{\sigma p} + \frac{(1-\gamma)(1+ \frac{1}{r})}{\sigma} \Big) \\
 \tilde{a}_2(p,r) = \Big(2 a  - \frac{a^2 }{\sigma} - \frac{a^2 }{\sigma p}   - (1-\gamma)(1+ \frac{1}{r}) \frac{a^2 }{\sigma}  \Big).
 \end{aligned}     
 \end{equation}

 If $a < \tilde{a}_1(p,r)$, then $\tilde{a}_2(p,r) >0$. Define $\beta_3(p,r)$ as 

\begin{equation}
    \begin{aligned}
        \beta_3(p,r) \triangleq (p + (1-\gamma)(1+r)) .
    \end{aligned}
\end{equation}

From Lemma~\ref{lem:ex_bound_lemma}, $\exists\  (p(\epsilon),r(\epsilon))$ such that  
\begin{equation}
\begin{aligned}
\Big(  p(\epsilon)  +(1-\gamma)(1+r(\epsilon)) \Big) &= 1 - \delta < 1, &&\\
\tilde{a}_1(p(\epsilon),r(\epsilon)) &>\zeta -\epsilon. \\
\end{aligned}
\end{equation}

 Choosing $(p,r) = (p(\epsilon),r(\epsilon))$, \eqref{eq:cf1} becomes
\begin{equation}\label{eq:cf2}
\begin{aligned}
||\hxtt - x^* ||^2 + ||m_{t+1}||^2  &\leq   ||\hxt - x^*||^2 + ||m_t||^2 \\ 
& - \alpha_t e_t \tilde{a}_2(p(\epsilon),r(\epsilon)).
\end{aligned}    
\end{equation} 

Note that if $ a \leq \zeta-\epsilon$, then $ a<\tilde{a}_1(p(\epsilon),r(\epsilon))$ and this implies $\tilde{a}_2(p(\epsilon),r(\epsilon))>0$.

Bounding $\alpha_t$ by  $ \alpha_{\mathrm{min}} \rho \leq \alpha_t $,

\begin{equation}\label{eq:cf4}
\begin{aligned}
  \alpha_{\mathrm{min}} \rho (f_{i_t}(x_t) -f_{i_t}(x^*)) \tilde{a}_2(p(\epsilon),r(\epsilon)) &\leq   ||\hxt - x^*||^2 - ||\hxtt - x^* ||^2\\ &+ ||m_t||^2 - ||m_{t+1}||^2.\\ 
\end{aligned}    
\end{equation}
 Taking expectation $E_{i_t}$ with respect to the data point $i_t$, conditioned on the entire past therein, 

\begin{equation}\label{eq:cf5}
\begin{aligned}
  \alpha_{\mathrm{min}} \rho (f(x_t) -f(x^*))\tilde{a}_2(p(\epsilon),r(\epsilon)) &\leq   ||\hxt - x^*||^2 - E_{i_t}[||\hxtt - x^* ||^2]\\ &+ ||m_t||^2 - E_{i_t}[||m_{t+1}||^2].\\ 
\end{aligned}    
\end{equation}
Taking expectation with respect to the entire process, and averaging over the time horizon $T$, 


\begin{equation}\label{eq:cf7}
\begin{aligned}
  \alpha_{\mathrm{min}} \rho \frac{1}{{T}} \sum_{t=0}^{T-1}(E[f(x_t)] -E[f(x^*)]) \tilde{a}_2(p(\epsilon),r(\epsilon)) &\leq \frac{1}{T} \sum_{t=0}^{T-1} ( E[||\hxt - x^*||^2] - E[||\hxtt - x^* ||^2])\\ &\hspace{2em}+ \frac{1}{T} \sum_{t=0}^{T-1} (E[||m_t||^2] - E[||m_{t+1}||^2])\\ 
  &\leq \frac{1}{T}  ( E[||\hat{x}_{0} - x^*||^2] - E[||\hat{x}_{T} - x^* ||^2])\\ & \hspace{2em}+ \frac{1}{T}  (E[||m_0||^2] - E[||m_{T}||^2]).
\end{aligned}    
\end{equation}

Since $f$ is convex, by Jensen's inequality, $$E[f(\frac{1}{T} {\sum_{t=0}^{T-1}x_t} )]   \leq \frac{1}{T} \sum_{t=0}^{T-1} E[f(x_t)].$$

Thus, 
\begin{equation}\label{eq:cf9}
\begin{aligned}
  \alpha_{\mathrm{min}} \rho (E[f(\frac{1}{T}{\sum_{t=0}^{T-1}x_t} )] -E[f(x^*)]) \tilde{a}_2(p(\epsilon),r(\epsilon)) &\leq \frac{1}{T}  ( E[||\hat{x}_{0} - x^*||^2] )\\ &+ \frac{1}{T}  (E[||m_0||^2] ).\\ 
\end{aligned}    
\end{equation}

Since $\hat{x}_0 = x_0$ and $m_0 = 0$, \eqref{eq:cf9} simplifies to

\begin{equation}\label{eq:cf9_cp}
\begin{aligned}
   \alpha_{\mathrm{min}} \rho (E[f(\frac{1}{T}{\sum_{t=0}^{T-1}x_t} )] -E[f(x^*)]) \tilde{a}_2(p(\epsilon),r(\epsilon))&\leq \frac{1}{T}  ( E[||{x}_{0} - x^*||^2] ).\\
\end{aligned}    
\end{equation}

Hence proved.
\end{proof}

\subsection{Proof of Theorem~\ref{thm:non_conv_asss}}

Let $f_i$ be $L_i$ smooth and let $f_i$ $i\in[n]$ satisfy the strong growth condition \eqref{eq:sgc}. Then, there exists $\hat{a},\hat{\alpha}$ such that for $0< a\leq \hat{a}$ and $\alpha_{\mathrm{max}} \leq \hat{\alpha}$ , 

\begin{equation}
\begin{aligned}
\frac{1}{T} \sum_{t=0}^{T-1} E[\norm{\nabla f(x_t)}^2] \leq \frac{(E[f({x}_{0})] - E[f(\hat{x}_{T})])}{\delta T}  
\end{aligned}    
\end{equation}

where $\hat{x}_T$ is a perturbed iterate~\cite{Mania2017PerturbedIA} obtained from $\{x_i\}_{i=0}^{T-1}$ and

\begin{equation}
\begin{aligned}
    \delta &= [(\eta_{\mathrm{max}} + \frac{\eta_{\mathrm{min}} p}{1+p} ) - (\nu (\eta_{\mathrm{max}} - \eta_{\mathrm{min}}) + \nu L \eta_{\mathrm{max}}^2 + (\nu \eta_{\mathrm{max}}^2 \theta(1-\gamma)(1+\frac{1}{r})) )] \\
    \hat{a} &= \min{\{\frac{[(\frac{p}{p+1}) +1]}{\tilde{\alpha}_{\mathrm{min}} v (L + \theta G)}, \frac{\theta \epsilon}{\beta L^2 + \td{\alpha}_{\mathrm{min}} p L^2} \}} \\
    \hat{\alpha} &= \frac{-(\nu -1) + \sqrt{(\nu-1)^2 + \Big(4\nu(L + \theta G)\hat{a}(\nu + \frac{p}{1+p})\tilde{\alpha}_{\mathrm{min}}\Big)}}{2\hat{a}(\nu L + \nu \theta G)} \\
    G &\triangleq \theta (1-\gamma)(1 + \frac{1}{r})\\
    L &= \frac{\sum_{i=1}^{n} L_i}{n}
\end{aligned}    
\end{equation}

for any $p,r,\theta >0$, $\tilde{\alpha}_{\mathrm{min}} = \frac{2(1-\sigma)}{L_{\mathrm{max}}}$, $L_{\mathrm{max}} = \max_{i}{\{L_i\}}$, $\epsilon< \gamma$ and $\nu$ is the strong growth constant.

\begin{proof}

By the perturbed iterate analysis framework, let $\{\hat{x}_t\}_{t\geq 0}$ be the virtual sequence generated as 

\begin{equation}\label{eq:th9_per}
\begin{aligned}
\hat{x}_{t+1} = \hat{x}_t - \eta_t \nabla f_{i_t}(x_t).
\end{aligned}\end{equation}

Since the gradient of $f$ is $L$-Lipschitz smooth,  

\begin{equation}\label{eq:th9_lip}
\begin{aligned}
f(\hat{x}_{t+1}) \leq f(\hat{x}_t) + \langle \nabla f(\hat{x}_t) , \hat{x}_{t+1} - \hat{x}_t \rangle + \frac{L}{2} \norm{\hat{x}_{t+1} - \hat{x}_t}^2.     
\end{aligned}    
\end{equation}

Substituting \eqref{eq:th9_per}, 
\begin{equation}\begin{aligned}
f(\hat{x}_{t+1}) \leq f(\hat{x}_t) + \langle \nabla f(\hat{x}_t) ,-\eta_t \nabla f_{i_t}(x_t))\rangle + \frac{L}{2} \norm{\eta_t \nabla f_{i_t}({x}_t)}^2. 
\end{aligned}\end{equation}

Using the identity $-2\langle a,b \rangle = \norm{a-b}^2 - \norm{a}^2 - \norm{b}^2$,

\begin{equation}\begin{aligned}
2(f(\hat{x}_{t+1}) - f(\hat{x}_t)) &\leq   \eta_t (\norm{\nabla f(\hat{x}_t) - \nabla f_{i_t}(x_t)}^2 - \norm{\nabla f(\hat{x}_t)}^2 - \norm{\nabla f_{i_t}(x_t)}^2) \\
&+ L \norm{\eta_t \nabla f_{i_t}({x}_t)}^2. 
\end{aligned}\end{equation}

The step-size $\eta_t \in [\eta_{\mathrm{min}},\eta_{\mathrm{max}}]$. By applying these bounds and taking expectation $E_{i_t}$ with respect to the data point $i_t$ conditioned on $x_t,i_{t-1}$ and the entire past therein, 

\begin{equation}\label{eq:th9_maxb}
\begin{aligned}
2 E_{i_t}[f(\hat{x}_{t+1}) - f(\hat{x}_t)] &\leq  \eta_{\mathrm{max}} E_{i_t}[||\nabla f(\hat{x}_t) - \nabla f_{i_t}(x_t) ||^2] &&\\
&  - \eta_{\mathrm{min}} ||\nabla f(\hat{x}_t)||^2 -  \eta_{\mathrm{min}} E_{i_t}[||\nabla f_{i_t}(x_t)||^2] &&\\
&+ L \eta_{\mathrm{max}}^2 E_{i_t}[|| \nabla f_{i_t}(x_t)||^2].
\end{aligned}\end{equation}

As $E_{i_t}[\nabla f_{i_t}(x)] = \nabla f(x)$, 

\begin{equation}\label{eq:th9_quad}
\begin{aligned}
E_{i_t}[||\nabla f(\hat{x}_t) - \nabla f_{i_t}(x_t) ||^2] &= ||\nabla f(\hat{x}_t)||^2 + E_{i_t}[||\nabla f_{i_t}(x_t)||^2]  &&\\
&-2\langle \nabla f(\hat{x}_t) , \nabla f(x_t)\rangle .
\end{aligned}
\end{equation}

Substituting \eqref{eq:th9_quad} in \ref{eq:th9_maxb}, 

\begin{equation}\begin{aligned}
2 E_{i_t}[f(\hat{x}_{t+1}) - f(\hat{x}_t)] &\leq  \eta_{\mathrm{max}} ( ||\nabla f(\hat{x}_t)||^2 + E_{i_t}[||\nabla f_{i_t}(x_t)||^2]) &&\\
&-\eta_{\mathrm{max}} 2\langle \nabla f(\hat{x}_t) , \nabla f(x_t)\rangle &&\\
&  -\eta_{\mathrm{min}} ||\nabla f(\hat{x}_t)||^2 - \eta_{\mathrm{min}} E_{i_t}[||\nabla f_{i_t}(x_t)||^2]\\
&+ L \eta_{\mathrm{max}}^2 E_{i_t}[|| \nabla f_{i_t}({x}_t)||^2].
\end{aligned}\end{equation}

Since,
\begin{equation}\begin{aligned}
-2 \eta_{\mathrm{max}} \langle \nabla f(\hat{x}_t), \nabla f(x_t) \rangle = \eta_{\mathrm{max}} \norm{\nabla f(\hat{x}_t)-  \nabla f(x_t)}^2 - \eta_{\mathrm{max}} \norm{\nabla f(\hat{x}_t)}^2 - \eta_{\mathrm{max}}\norm{\nabla f(x_t)}^2,
\end{aligned}\end{equation}

\begin{equation}\begin{aligned}
2 E_{i_t}[f(\hat{x}_{t+1}) - f(\hat{x}_t)] &\leq  \eta_{\mathrm{max}} ( ||\nabla f(\hat{x}_t)||^2 + E_{i_t}[||\nabla f_{i_t}(x_t)||^2])  &&\\
&+ \eta_{\mathrm{max}} \norm{\nabla f(\hat{x}_t)-  \nabla f(x_t)}^2 - \eta_{\mathrm{max}} \norm{\nabla f(\hat{x}_t)}^2 &&\\
&- \eta_{\mathrm{max}}\norm{\nabla f(x_t)}^2 
 -\eta_{\mathrm{min}} ||\nabla f(\hat{x}_t)||^2 &&\\
 &- \eta_{\mathrm{min}} E_{i_t}[||\nabla f_{i_t}(x_t)||^2] + L \eta_{\mathrm{max}}^2 E_{i_t}[|| \nabla f_{i_t}({x}_t)||^2]
\end{aligned}\end{equation}

From Lipschitz smoothness of $\nabla f(x)$, 

\begin{equation}
||\nabla f(\hat{x}_t) - \nabla f(x_t)|| \leq L ||\hat{x}_t - x_t ||.    
\end{equation}

Substituting the Lipschitz-smoothness property, 
\begin{equation}\label{eq:th9_eqlb}
\begin{aligned}
2 E_{i_t}[f(\hat{x}_{t+1}) - f(\hat{x}_t)] &\leq (\eta_{\mathrm{max}} - \eta_{\mathrm{min}} + L \eta_{\mathrm{max}}^2) E_{i_t}[||\nabla f_{i_t}(x_t) ||^2] + \eta_{\mathrm{max}} L^2 [||\hat{x}_t - x_t||^2]  &&\\
&\hspace{1cm}- \eta_{\mathrm{max}} ||\nabla f(x_t) ||^2 - \eta_{\mathrm{min}} ||\nabla f(\hat{x}_t) ||^2.
\end{aligned}\end{equation}

For any $p>0$, from  Lemma~\ref{lem:math_prelim_tri},

\begin{equation}\begin{aligned}
\norm{a+b}^2 \leq (1+p)\norm{a}^2 + (1+ \frac{1}{p}) \norm{b}^2.
\end{aligned}\end{equation}

Using this identity, 
\begin{equation}
\begin{aligned}
\norm{\nabla f(x_t)}^2 \leq (1+p) \norm{\nabla f(x_t) - \nabla f(\hat{x}_t)}^2  + (1+\frac{1}{p}) \norm{\nabla f(\hat{x}_t)}^2.
\end{aligned}\end{equation}

Since $f(x)$ has $L$-Lipschitz gradients,

\begin{equation}\begin{aligned}
\norm{\nabla f(x_t)}^2 \leq (1+p) L^2 \norm{x_t - \hat{x}_t}^2  + (1+\frac{1}{p}) \norm{\nabla f(\hat{x}_t)}^2,
\end{aligned}\end{equation}

\begin{equation}\label{eq:th9_lipg}
\begin{aligned}
\frac{p}{1 + p} \norm{\nabla f(x_t)}^2 - p L^2 \norm{\hat{x}_t - x_t}^2  \leq \norm{\nabla f(\hat{x}_t)}^2.
\end{aligned}\end{equation}

Substituting \eqref{eq:th9_lipg} in \eqref{eq:th9_eqlb},
\begin{equation}\begin{aligned}
2 E_{i_t}[f(\hat{x}_{t+1}) - f(\hat{x}_t)] &\leq (\eta_{\mathrm{max}} - \eta_{\mathrm{min}} + L \eta_{\mathrm{max}}^2) E_{i_t}[||\nabla f_{i_t}(x_t) ||^2] + \eta_{\mathrm{max}} L^2 [||\hat{x}_t - x_t||^2]  &&\\
&- \eta_{\mathrm{max}} ||\nabla f(x_t) ||^2 - \eta_{\mathrm{min}} \big(\frac{p}{1 + p} \norm{\nabla f(x_t)}^2 - p L^2 \norm{\hat{x}_t - x_t}^2\big),
\end{aligned}\end{equation}

\begin{equation}\begin{aligned}
2E_{i_t}[f(\hat{x}_{t+1}) - f(\hat{x}_t)] &\leq (\eta_{\mathrm{max}}-\eta_{\mathrm{min}} + L \eta_{\mathrm{max}}^2) E_{i_t}[\norm{\nabla f_{i_t}(x_t)}^2] &&\\
&+ (\eta_{\mathrm{max}} L^2 + \eta_{\mathrm{min}} p L^2) (\norm{\hat{x}_t -x_t}^2)  &&\\
&- (\eta_{\mathrm{max}} + \frac{\eta_{\mathrm{min}} p }{1+ p})(\norm{\nabla f(x_t)}^2),  
\end{aligned}\end{equation}

\begin{equation}\begin{aligned}
(\eta_{\mathrm{max}} + \frac{\eta_{\mathrm{min}} p }{1+ p})(\norm{\nabla f(x_t)}^2) &\leq 2E_{i_t}[f(\hat{x}_t) -f(\hat{x}_{t+1})]&&\\
&+ (\eta_{\mathrm{max}}-\eta_{\mathrm{min}} + L \eta_{\mathrm{max}}^2) E_{i_t}[\norm{\nabla f_{i_t}(x_t)}^2] &&\\
& + (\eta_{\mathrm{max}} L^2 + \eta_{\mathrm{min}} p L^2) E_{i_t}[\norm{\hat{x}_t -x_t}^2].
\end{aligned}\end{equation}

By strong growth condition, $E_{i_t}[\norm{\nabla f_{i_t}(x)}^2] \leq \nu \norm{\nabla f(x)}^2 $ and thus,

\begin{equation}\begin{aligned}
(\eta_{\mathrm{max}} + \frac{\eta_{\mathrm{min}} p }{1+ p})(\norm{\nabla f(x_t)}^2) &\leq 2E_{i_t}[f(\hat{x}_t) -f(\hat{x}_{t+1})] &&\\
& + (\eta_{\mathrm{max}}-\eta_{\mathrm{min}} + L \eta_{\mathrm{max}}^2) \nu \norm{\nabla f(x_t)}^2 +&&\\
& + (\eta_{\mathrm{max}} L^2 + \eta_{\mathrm{min}} p L^2) (\norm{\hat{x}_t -x_t}^2),
\end{aligned}\end{equation}

\begin{equation}\label{eq:th9_addeq1}
\begin{aligned}
(\eta_{\mathrm{max}} + \frac{\eta_{\mathrm{min}} p }{1+ p} - (\eta_{\mathrm{max}}-\eta_{\mathrm{min}} + L \eta_{\mathrm{max}}^2) \nu)(\norm{\nabla f(x_t)}^2) &\leq 2E_{i_t}[f(\hat{x}_t) -f(\hat{x}_{t+1})]  &&\\
&\hspace{-2em} + (\eta_{\mathrm{max}} L^2 + \eta_{\mathrm{min}} p L^2) (\norm{\hat{x}_t -x_t}^2).
\end{aligned}\end{equation}

From Lemmas~\ref{lem:math_prelim_error} and \ref{lem:math_prelim_compress}:

\begin{equation}\begin{aligned}
\norm{x- top_k(x)}^2 \leq (1-\gamma)\norm{x}^2,
\end{aligned}\end{equation}

\begin{equation}
\begin{aligned}
 \norm{m_t} = \norm{x_t - \hat{x}_{t}}.   
\end{aligned}    
\end{equation}

The error update equation of the CSGD-ASSS algorithm is

\begin{equation}\begin{aligned}
m_{t+1} = m_t + \eta_t \nabla f_{i_t}(x_t) - top_k(m_t + \eta_t \nabla f_{i_t}(x_t)).
\end{aligned}\end{equation}

Hence, 

\begin{equation}\label{eq:th9_addeq2_v2}
\begin{aligned}
\norm{m_{t+1}}^2 &= \norm{m_t + \eta_t \nabla f_{i_t}(x_t) - top_k(m_t + \eta_t \nabla f_{i_t}(x_t))}^2 &&\\
     &\leq (1-\gamma)\norm{m_t + \eta_t \nabla f_{i_t}(x_t)}^2. &&
\end{aligned}\end{equation}

From Lemma~\ref{lem:math_prelim_tri}, for any $r>0$, 

\begin{equation}\label{eq:th9_addeq2}
\begin{aligned}
\norm{m_{t+1}}^2
     &\leq (1-\gamma)(1+r) \norm{m_t}^2 + (1-\gamma)(1+\frac{1}{r}) \eta_t^2\norm{\nabla f_{i_t}(x_t)}^2\\
     &\leq (1-\gamma)(1+r) \norm{m_t}^2 + (1-\gamma)(1+\frac{1}{r}) \eta_{\mathrm{max}}^2\norm{\nabla f_{i_t}(x_t)}^2
\end{aligned}\end{equation}

Taking expectation w.r.t data-point $i_t$ and mutiplying \eqref{eq:th9_addeq2} by $\theta>0$ and adding to  \ref{eq:th9_addeq1},

\begin{equation}
\begin{aligned}
(\eta_{\mathrm{max}} + \frac{\eta_{\mathrm{min}} p }{1+ p} - &(\eta_{\mathrm{max}}-\eta_{\mathrm{min}} + L \eta_{\mathrm{max}}^2) \nu)(\norm{\nabla f(x_t)}^2) + \theta E_{i_t}[\norm{m_{t+1}}^2] &&\\
&\leq 2E_{i_t}[f(\hat{x}_t) -f(\hat{x}_{t+1})]  &&\\
 &+ (\eta_{\mathrm{max}} L^2 + \eta_{\mathrm{min}} p L^2) (\norm{\hat{x}_t -x_t}^2)  &&\\
 &+ \theta(1-\gamma)(1+r) \norm{m_t}^2 + \theta (1-\gamma)(1+\frac{1}{r}) \eta_{\mathrm{max}}^2 (E_{i_t}[\norm{\nabla f(x_t)}^2]).
\end{aligned}
\end{equation}

By strong growth condition, 

\begin{equation}
\begin{aligned}
(\eta_{\mathrm{max}} + \frac{\eta_{\mathrm{min}} p }{1+ p} - &(\eta_{\mathrm{max}}-\eta_{\mathrm{min}} + L \eta_{\mathrm{max}}^2) \nu)\norm{\nabla f(x_t)}^2 + \theta E_{i_t}[\norm{m_{t+1}}^2] &&\\
&\leq 2E_{i_t}[f(\hat{x}_t) -f(\hat{x}_{t+1})]  +&&\\
 &+ (\eta_{\mathrm{max}} L^2 + \eta_{\mathrm{min}} p L^2) \norm{\hat{x}_t -x_t}^2 + &&\\
 &+ \theta(1-\gamma)(1+r) \norm{m_t}^2 + \theta (1-\gamma)(1+\frac{1}{r}) \eta_{\mathrm{max}}^2 \nu \norm{\nabla f(x_t)}^2.
\end{aligned}
\end{equation}

Taking expectation w.r.t the entire process, 

\begin{equation}\label{eq:th9_delteq}
\begin{aligned}
&[(\eta_{\mathrm{max}} + \frac{\eta_{\mathrm{min}} p}{1+p} ) - (\nu (\eta_{\mathrm{max}} - \eta_{\mathrm{min}}) + \nu L \eta_{\mathrm{max}}^2 &&\\ &+ (\nu \eta_{\mathrm{max}}^2 \theta(1-\gamma)(1+\frac{1}{r})) )] E[\norm{\nabla f(x_t)}^2]  &&\\
 & \hspace{5em}\leq 2 E[f(\hat{x}_t) - f(\hat{x}_{t+1})] &&\\
 &\hspace{5em}+ \big( \eta_{\mathrm{max}} L^2 + \eta_{\mathrm{min}} p L^2 + \theta(1-\gamma)(1+r) \big) E[\norm{m_t}^2] - \theta E[\norm{m_{t+1}}^2].
\end{aligned}\end{equation}

Consider the term $\eta_{\mathrm{max}} L^2 + \eta_{\mathrm{min}} p L^2 + \theta(1-\gamma)(1+r)$ and an  $\epsilon < \gamma $. Then,

\begin{equation}
\begin{aligned}
(1-\gamma)(1+r) \leq 1- \gamma + r \leq (1-\epsilon)    
\end{aligned}    
\end{equation}
iff $r \leq \gamma - \epsilon$. Set

\begin{equation}
\begin{aligned}
&r \leq \gamma - \epsilon &&\\
&\eta_{\mathrm{max}} L^2 + \eta_{\mathrm{min}} p L^2 \leq \theta \epsilon. &&\\
\end{aligned}    
\end{equation}

 The inequality $\eta_{\mathrm{max}} L^2 + \eta_{\mathrm{min}} p L^2 \leq \theta \epsilon$ holds iff

\begin{equation}\label{eq:a_bound1}
\begin{aligned}
a \leq \frac{\theta \epsilon}{(\alpha_{\mathrm{max}} L^2 + \alpha_{\mathrm{min}} p L^2)}.    
\end{aligned}    
\end{equation}

Hence, setting $r = \gamma - \epsilon $ and $a$ satisfying \ref{eq:a_bound1}, implies

\begin{equation}
    \begin{aligned}
    \big( \eta_{\mathrm{max}} L^2 + \eta_{\mathrm{min}} p L^2 + \theta(1-\gamma)(1+r) \big) \leq \theta.
    \end{aligned}
\end{equation}
 
Let
\begin{equation}
\begin{aligned}
 \delta \triangleq [(\eta_{\mathrm{max}} + \frac{\eta_{\mathrm{min}} p}{1+p} ) - (\nu (\eta_{\mathrm{max}} - \eta_{\mathrm{min}}) + \nu L \eta_{\mathrm{max}}^2 + (\nu \eta_{\mathrm{max}}^2 \theta(1-\gamma)(1+\frac{1}{r})) )]  
\end{aligned}    
\end{equation}

and 
\begin{equation}
  \begin{aligned}
     G \triangleq \theta (1-\gamma)(1 + \frac{1}{r}) .
  \end{aligned}  
\end{equation}

If $\delta >0$, and $a$ satisfies~\eqref{eq:a_bound1}, \eqref{eq:th9_delteq} simplifies as 

\begin{equation}
\begin{aligned}
\delta E[\norm{\nabla f(x_t)}^2] \leq 2 E[f(\hat{x}_{t}) - f(\hat{x}_{t+1})] + \theta ( E[\norm{m_t}^2] - E[\norm{m_{t+1}}^2]).    
\end{aligned}    
\end{equation}

Summing over the horizon $T$ and averaging, 

\begin{equation}
\begin{aligned}
\frac{1}{T} \sum_{t=0}^{T-1} E[\norm{\nabla f(x_t)}^2] \leq 2 \frac{(E[f(\hat{x}_{0})] - E[f(\hat{x}_{T})])}{\delta T} + \frac{\theta(E[\norm{m_0}^2] - E[\norm{m_T}^2])}{\delta T} .   
\end{aligned}    
\end{equation}

Since $m_0 =0$ and $\hat{x}_0 = x_0$,

\begin{equation}
\begin{aligned}
\frac{1}{T} \sum_{t=0}^{T-1} E[\norm{\nabla f(x_t)}^2] \leq 2 \frac{(E[f(\hat{x}_{0})] - E[f(\hat{x}_{T})])}{\delta T}.
\end{aligned}    
\end{equation}

The conditions under which $\delta >0$, can be stated in 2 cases depending on the value of $\alpha_{\mathrm{max}}$.

\begin{enumerate}

\item Case 1: $\alpha_{\mathrm{max}} \leq \frac{2(1-\sigma)}{L_{\mathrm{max}}} $.

In this case $\alpha_{\mathrm{max}} = \alpha_{\mathrm{min}}$ by Lemma~\ref{lem:math_prelim_armijo_stop} and this implies $\eta_{\mathrm{max}} = \eta_{\mathrm{min}}$. Hence $\delta >0$ iff 

\begin{equation}\begin{aligned}
a \leq \frac{[(\frac{p}{p+1}) +1] }{\tilde{\alpha}_{\mathrm{min}} v (L + \theta G)}
\end{aligned}\end{equation}

\item Case 2: $ \alpha_{\mathrm{max}} > \frac{2(1-\sigma)}{L_{\mathrm{max}}} $. In this case $\delta >0 $ iff 

\begin{equation}
\begin{aligned}
\alpha_{\mathrm{max}} \leq  \frac{-(\nu -1) + \sqrt{(\nu-1)^2 + \Big(4\nu(L + \theta G)a(\nu + \frac{p}{1+p})\tilde{\alpha}_{\mathrm{min}} \rho\Big)}}{2a(\nu L + \nu \theta G)}
\end{aligned}    
\end{equation}

Let

\begin{equation}
\begin{aligned}
\mathrm{UB}(a) \triangleq \frac{-(\nu -1) + \sqrt{(\nu-1)^2 + \Big(4\nu(L + \theta G)a(\nu + \frac{p}{1+p})\tilde{\alpha}_{\mathrm{min}}\rho\Big)}}{2a(\nu L + \nu \theta G)}    
\end{aligned}    
\end{equation}

\end{enumerate}

An interval $\alpha_{\mathrm{max}} \in (\frac{2(1-\sigma)}{L_{\mathrm{max}}} , \mathrm{UB}(a)]$ exists if $\frac{2(1-\sigma)}{L_{\mathrm{max}}} \leq \mathrm{UB}(a)$. The bound $\mathrm{UB}(a)$ is a monotonically decreasing function of $a$ and the interval exists if 

\begin{equation}\label{eq:a_bound2}
\begin{aligned}
a \leq \frac{[(\frac{p}{p+1}) +1]}{\tilde{\alpha}_{\mathrm{min}} v (L + \theta G)}.
\end{aligned}    
\end{equation}

Combining the bounds on $a$ for $\delta >0$ and $\frac{1}{T}$ convergence from \eqref{eq:a_bound1},\eqref{eq:a_bound2},

\begin{equation}
\begin{aligned}
a \leq \min{\{\frac{[(\frac{p}{p+1}) +1]}{\tilde{\alpha}_{\mathrm{min}} v (L + \theta G)}, \frac{\theta \epsilon}{\beta L^2 + \td{\alpha}_{\mathrm{min}} p L^2} \}}.    
\end{aligned}    
\end{equation}

Hence, in the statement of the theorem, 
\begin{equation}
    \begin{aligned}
        \hat{a} &= \min{\{\frac{[(\frac{p}{p+1}) +1]}{\tilde{\alpha}_{\mathrm{min}} v (L + \theta G)}, \frac{\theta \epsilon}{\beta L^2 + \td{\alpha}_{\mathrm{min}} p L^2} \}}, \\
    \end{aligned}
\end{equation}

and if $a = \hat{a}$,  then 
\begin{equation}
    \begin{aligned}
        \hat{\alpha} &= \mathrm{UB}(\hat{a}).
    \end{aligned}
\end{equation}
Hence proved.

\end{proof}

\section{Distributed Compressed SGD with Armijo Step-Size Search~(DCSGD-ASSS)}

The CSGD-ASSS algorithm can be extended to the distributed setting with the same convergence rates as the central/single worker node setting. In this section, we present the algorithm and convergence results for DCSGD-ASSS. 

\subsection{Notation}
\begin{flushleft}
In the distributed setting, the superscript indicates the corresponding worker node. For instance, $g_{t}^{(k)}$ is the compressed gradient of worker node $k$ at time $t$. We assume for simplicity, the dataset is equally split among the workers and $n= NM $, where $n$ is the total number of data points, $M$ is the number of data points in each worker and $N$ is the number of worker nodes. The $i$~th data point at worker $k$ is $f_{i}^{(k)}(x)$.
\end{flushleft}

\subsection{Algorithm}
\begin{flushleft}
Consider the distributed setting in which there is one central node and $N$ worker nodes. The function to be minimized is 
\begin{equation}
\begin{aligned}
f(x)= \frac{1}{N} {\sum_{k=1}^{N} f^{(k)}(x)}
\end{aligned}    
\end{equation}
where $ f^{(k)}(x) = \frac{1}{M} {\sum_{i=1}^{M} f_i^{(k)}(x)} $.
We consider a distributed algorithm consistent with \cite{aji_hea,dan_convergence}. Each worker node samples a data point $f_{i_t}^{(k)}(x)$, searches for a step-size $\alpha_t^{(k)}$ with the Armijo step-size search algorithm,  computes $a \stt^{(k)} g_{t}^{(k)} $, and transmits it to the central node, where $g_{t}^{(k)}$ is the compressed gradient. The central node averages the compressed gradients and computes the descent step $x_{t+1} = x_{t} - \big(\sum_{k=1}^{N} a \stt^{(k)} g_{t}^{(k)} /N \big)$ and transmits $x_{t+1}$  to the worker nodes.
\end{flushleft}

\begin{algorithm}[H]
\caption{Distributed Compressed SGD with Armijo Step-Size Search (DCSGD-ASSS) }
\label{alg:distri_armijo}
\begin{algorithmic}[1]
\State \textbf{At worker node $(k),\ k\in [N]$ :}
\For{$t = 0,\cdots, T-1$}
    \State Sample data $i_t^{(k)}$
    \State $\alpha_{t}^{(k)} \leftarrow \text{Armijo step-size search}(f_{i_t}^{(k)},\alpha_{\mathrm{max}},x_t)$
    \State $g_{t}^{(k)} = top_k(m_t^{(k)} + a \alpha_t^{(k)} \nabla f_{i_t}^{(k)} (x_t) )$ \label{step:dist_armi_inde}
    \State Transmit $g_{t}^{(k)} $  to \textit{central node}
    \State $m_{t+1}^{(k)} = m_t^{(k)} + a \alpha_t^{(k)} \nabla f_{i_t}^{(k)} (x_t) - g_t^{(k)}$ 
\EndFor
\State \textbf{At central node:} 
    \State $x_{t+1} = x_t - \big(\sum_{k=1}^{N}  g_{t}^{(k)} /N \big)$
    \State Transmit $x_{t+1}$ to worker nodes 
\end{algorithmic}
\end{algorithm}

\subsection{Results}
\begin{flushleft}
The convergence results for DCSGD-ASSS are proved using the perturbed iterate analysis method. Let $\hat{x}_0 = x_0$ and $\hat{x}_{t+1} = \hat{x}_t - \frac{1}{N} \sum_{k=1}^N  {a \alpha_t^{(k)} \nabla f_{i_t}^{(k)}(x_t) }$. The Lemma~\ref{lem:memory} is used in the proof of convergence.

\end{flushleft}

\begin{lemma}\label{lem:memory}
 For any time $t$, 
 \begin{equation}
     x_{t} - \hat{x}_{t} = \frac{1}{N}\sum_{k=1}^N m_t^{(k)}. 
 \end{equation}
\end{lemma}
\begin{proof}
 We prove this by induction. Note that $m_0^{(k)} = 0$ for all $k$. For $t = 1$, 
 \begin{align}
     x_1 &= x_0 - \frac{1}{N} \sum_{k=1}^N  {top_k(a \alpha_0^{(k)} \nabla f_{i_0}^{(k)} (x_0))}, \\
     \hat{x}_1 &= \hat{x}_0 -  \frac{1}{N}\sum_{k=1}^N {a \alpha_0^{(k)} \nabla f_{i_0}^{(k)} (x_0)}.
 \end{align}
 Then, 
 \begin{equation}
     x_1 - \hat{x}_1 = \frac{1}{N} \sum_{k=1}^N ({a \alpha_0^{(k)} \nabla f_{i_0}^{(k)} (x_0) - top_k(a \alpha_0^{(k)} \nabla f_{i_0}^{(k)} (x_0))})= \frac{1}{N} \sum_{k=1}^N  {m_1^{(k)}}.
 \end{equation}
 Let $x_{t} - \hat{x}_{t} = \frac{1}{N}\sum_{k=1}^N m_t^{(k)}$. For $t+1$, 
 \begin{align}
     x_{t+1} &= x_t - \frac{1}{N} \sum_{k=1}^N {top_k(m_t^{(k)} + a \alpha_t^{(k)} \nabla f_{i_t}^{(k)} (x_t))}, \\
     \hat{x}_{t+1} &= \hat{x}_t -  \frac{1}{N}\sum_{k=1}^N {a \alpha_t^{(k)} \nabla f_{i_t}^{(k)} (x_t)}.
 \end{align}
 
 Thus, 
 \begin{equation}
 \begin{aligned}
     x_{t+1} - \hat{x}_{t+1} &= \frac{1}{N}\sum_{k=1}^N \big({m_t^{(k)} + a \alpha_t^{(k)} \nabla f_{i_t}^{(k)} (x_t)) -top_k(m_t^{(k)} + a \alpha_t^{(k)} \nabla f_{i_t}^{(k)} (x_t))}\big) &&\\
     &= \frac{1}{N}\sum_{k=1}^N {m_{t+1}^{(k)}}.&&\\
      \end{aligned}
\end{equation}
\end{proof}

\begin{theorem}
(Distributed-Convex) Let $f_{i}^{(k)}(x)$ be convex and $L_i^{(k)}$ smooth for $i\in[M]$, $k \in [N]$, and $f_{i}^{(k)}$ satisfy the interpolation condition. Then there exists $\hat{a}$ such that if $a < \hat{a}$, the distributed  compressed SGD with Armijo step-size search for $\sigma \in (0,1)$ and scale factor $a$ satisfies
\begin{equation}
\begin{aligned}
  &E\left[f\left(\frac{1}{T}{\sum_{t=0}^{T-1}x_t} \right)\right] -E[f(x^*)] \\ &\leq  \frac{1}{\delta_2 T}  ( E[||\hat{x}_{0} - x^*||^2] )\\
\end{aligned}    
\end{equation}

where for any $0< \epsilon < \zeta$, $\zeta \triangleq \frac{\sigma \gamma}{(2-\gamma)}$ and $L_{\mathrm{max}} \triangleq \max_{i,k}L_i^{(k)}$,

\begin{equation}
\begin{aligned}
&\delta_2 \triangleq \rho \frac{2(1-\sigma)}{L_{\mathrm{max}}} \Big(2 a  - \frac{a^2 }{\sigma} - \frac{a^2 }{\sigma p(\epsilon)}   - (1-\gamma)\left(1+ \frac{1}{r(\epsilon)}\right) \frac{a^2 }{\sigma}  \Big),\\
& \hat{a} = (\zeta-\epsilon). \\
\end{aligned}    
\end{equation}

The values $p(\epsilon),r(\epsilon)$ are as stated in Lemma~\ref{lem:ex_bound_lemma}.

\end{theorem}

\begin{proof}

We have that

\begin{equation}
    \begin{aligned}
        \|\hat{x}_{t+1} - x^*\|^2 &= \|\hat{x}_t - x^* - \frac{1}{N} \sum_{k=1}^N {a \alpha_t^{(k)} \nabla f_{i_t}^{(k)}(x_t) }\|^2 \\
        &= \|\hat{x}_t - x^* \|^2 + \frac{1}{N}\|\sum_{k=1}^N {a \alpha_t^{(k)} \nabla f_{i_t}^{(k)}(x_t) }\|^2 - 2 \langle \hat{x}_t - x^*, \frac{1}{N} \sum_{k=1}^N {a \alpha_t^{(k)} \nabla f_{i_t}^{(k)}(x_t) } \rangle \\
        &\leq \|\hat{x}_t - x^* \|^2 + \frac{1}{N} \sum_{k=1}^N {a^2 (\alpha_t^{(k)})^2 \|\nabla f_{i_t}^{(k)}(x_t)\|^2 } - \frac{2}{N} \sum_{k=1}^N \langle \hat{x}_t - x^*, {a \alpha_t^{(k)} \nabla f_{i_t}^{(k)}(x_t) } \rangle \\
        &= \|\hat{x}_t - x^* \|^2 +\frac{1}{N} \sum_{k=1}^N {a^2 (\alpha_t^{(k)})^2 \|\nabla f_{i_t}^{(k)}(x_t)\|^2 }\\ & + \frac{2}{N} \sum_{k=1}^N \langle x_t - \hat{x}_t, {a \alpha_t^{(k)} \nabla f_{i_t}^{(k)}(x_t) } \rangle + \frac{2}{N} \sum_{k=1}^N \langle  x^* - {x}_t , {a \alpha_t^{(k)} \nabla f_{i_t}^{(k)}(x_t) } \rangle,
    \end{aligned}
\end{equation}

where the inequality arises from Jensen's inequality. 

Since each $f_{i_t}^{(k)}$ is convex, 

\begin{equation}
    \langle x^* - x_t , a \alpha_{t}^{(k)} \nabla f_{i_t}^{(k)}(x_t) \rangle \leq a \alpha_{t}^{(k)}(f_{i_t}^{(k)}(x^*) - f_{i_t}^{(k)}(x_t)).
\end{equation}

From Lemma~\ref{lem:math_prelim_tri}, for any $p >0$,

\begin{equation}
     \langle x_t - \hat{x}_t, \frac{2}{N} \sum_{k=1}^N {a \alpha_t^{(k)} \nabla f_{i_t}^{(k)}(x_t) } \rangle \leq p \|x_t - \hat{x}_t\|^2 + \frac{1}{p}  \|\frac{1}{N}\sum_{k=1}^N {a^2 (\alpha_t^{(k)})^2 \nabla f_{i_t}^{(k)}(x_t)}\|^2.
\end{equation}

From Jensen's inequality, 
\begin{equation}
     \langle x_t - \hat{x}_t, \frac{2}{N} \sum_{k=1}^N {a \alpha_t^{(k)} \nabla f_{i_t}^{(k)}(x_t) } \rangle \leq p \|x_t - \hat{x}_t\|^2 + \frac{1}{p N} \sum_{k=1}^Na^2 (\alpha_t^{(k)})^2 \|\nabla f_{i_t}^{(k)}(x_t)\|^2.
\end{equation}

Thus, 
\begin{align}
    \|\hat{x}_{t+1} - x^*\|^2 &\leq \|\hat{x}_t - x^* \|^2 + \frac{1}{N} \sum_{k=1}^N \big({a^2 (\alpha_t^{(k)})^2 \|\nabla f_{i_t}^{(k)}(x_t)\|^2 }\big) +  p \|x_t - \hat{x}_t\|^2 \\
    &+ \frac{1}{p N} \sum_{k=1}^N a^2 (\alpha_t^{(k)})^2 \|\nabla f_{i_t}^{(k)}(x_t)\|^2 + \frac{2}{N} \sum_{k=1}^N a \alpha_{t}^{(k)}(f_{i_t}^{(k)}(x^*) - f_{i_t}^{(k)}(x_t)) \\
    &= \|\hat{x}_t - x^* \|^2 + p \| \frac{1}{N}\sum_{k=1}^N{m_t^{(k)}}\|^2 + (1 + \frac{1}{p}) \frac{1}{N} \sum_{k=1}^N {a^2 (\alpha_t^{(k)})^2 \|\nabla f_{i_t}^{(k)}(x_t)\|^2 } &&\\
    &\hspace{3em}- \frac{2}{N} \sum_{k=1}^N a \alpha_{t}^{(k)} e_t^{(k)} \label{eq:lemma3} \\ 
    &\leq \|\hat{x}_t - x^* \|^2 + \frac{p}{N} \sum_{k=1}^N \|m_t^{(k)}\|^2 + \frac{(1 + \frac{1}{p})}{N} \sum_{k=1}^N {a^2 (\alpha_t^{(k)})^2 \|\nabla f_{i_t}^{(k)}(x_t)\|^2 } &&\\
    &\hspace{3em}- \frac{2}{N} \sum_{k=1}^N a \alpha_{t}^{(k)} e_t^{(k)}, \label{eq:jensen1}
\end{align}
where $e_t^{(k)} \triangleq f_{i_t}^{(k)}(x_t) - f_{i_t}^{(k)}(x^*)$, \eqref{eq:lemma3} is from Lemma \ref{lem:memory} and \eqref{eq:jensen1} is from Jensen's inequality.

We have from the compression property~(Lemma~\ref{lem:math_prelim_error}) and Lemma~\ref{lem:math_prelim_tri2} that for each $k$ and some $r > 0$,
\begin{equation}
    \begin{aligned}
||m_{t+1}^{(k)}||^2 &\leq (1-\gamma) ||m_t^{(k)} + a \alpha_{t}^{(k)} \nabla f^{(k)}_{i_t}(x_t) ||^2 \\
 & \leq (1- \gamma) \big( (1+r) ||m_t^{(k)}||^2 + (1+ \frac{1}{r}) a^2 (\alpha_{t}^{(k)})^2 ||\nabla f^{(k)}_{i_t}(x_t) ||^2 \big)\\
 & \leq (1-\gamma)(1+r) ||m^{(k)}_t||^2 + (1-\gamma) (1+ \frac{1}{r}) a^2 (\alpha_{t}^{(k)})^2 ||\nabla f^{(k)}_{i_t}(x_t)||^2.
\end{aligned} 
\end{equation}

Averaging the above over all $k$~(the workers),
\begin{equation}\label{eq:avgmem}
    \frac{1}{N}\sum_{k=1}^N ||m_{t+1}^{(k)}||^2 \leq \frac{(1-\gamma)(1+r)}{N} \sum_{k=1}^N {||m^{(k)}_t||^2} + \frac{(1-\gamma) (1+ \frac{1}{r})}{N} \sum_{k=1}^N {a^2 (\alpha_{t}^{(k)})^2 ||\nabla f^{(k)}_{i_t}(x_t)||^2}.
\end{equation}

Adding \eqref{eq:avgmem} and \eqref{eq:jensen1}, we get, 
\begin{equation}\label{eq:convexsum}
\begin{aligned}
    \|\hat{x}_{t+1} - x^*\|^2 + \frac{1}{N}\sum_{k=1}^N {||m_{t+1}^{(k)}||^2} &\leq \|\hat{x}_t - x^* \|^2 + \frac{p}{N} \sum_{k=1}^N \|m_t^{(k)}\|^2 &&\\
    &+ (\frac{p+1}{pN}) \sum_{k=1}^N {a^2 (\alpha_t^{(k)})^2 \|\nabla f_{i_t}^{(k)}(x_t)\|^2 }
    - \frac{2}{N} \sum_{k=1}^N a \alpha_{t}^{(k)} e_t^{(k)} \\ &+ \frac{(1-\gamma)(1+r)}{N} \sum_{k=1}^N {||m^{(k)}_t||^2} \\
    &+ (1-\gamma) (1+ \frac{1}{r})\frac{1}{N} \sum_{k=1}^N {a^2 (\alpha_{t}^{(k)})^2 ||\nabla f^{(k)}_{i_t}(x_t)||^2}.
\end{aligned}
\end{equation}

For any worker $k$, from the Armijo step-size search stopping condition,
\begin{align}
    \|\nabla f_{i_t}^{(k)}(x_t)\|^2  \leq \frac{1}{\sigma  \alpha_{t}^{(k)}}  ( f_{i_t}^{(k)}(x_t) - f_{i_t}^{(k)}(x_t - \alpha_{t}^{(k)} \nabla f_{i_t}^{(k)}(x_t))).
\end{align}
From the interpolation condition we have that $f_{i_t}^{(k)}(x_t - \alpha_{t}^{(k)} \nabla f_{i_t}^{(k)}(x_t)) \geq f_{i_t}^{(k)}(x^*)$. Thus,  
\begin{equation}
    a^2 (\alpha_{t}^{(k)})^2 \|\nabla f_{i_t}^{(k)}(x_t)\|^2 \leq \frac{a^2 \alpha_{t}^{(k)}}{\sigma} (f_{i_t}^{(k)}(x_t) - f_{i_t}^{(k)}(x^*)).
\end{equation}

Summing over all $k$ we get, 
\begin{equation}
     \sum_{k=1}^N a^2 (\alpha_{t}^{(k)})^2 \|\nabla f_{i_t}^{(k)}(x_t)\|^2 \leq \frac{1}{\sigma}\sum_{k=1}^N {a^2 \alpha_{t}^{(k)}} e_t^{(k)}.
\end{equation}

Using the above in \eqref{eq:convexsum}, 

\begin{equation}\label{eq:dist_p_r_subs}
    \begin{aligned}
        \|\hat{x}_{t+1} - x^*\|^2 + \sum_{k=1}^N \frac{||m_{t+1}^{(k)}||^2}{N} &\leq \|\hat{x}_t - x^* \|^2 + \frac{p + (1-\gamma)(1+r)}{N} \sum_{k=1}^N {||m^{(k)}_t||^2} \\
        &+ \frac{\frac{a^2}{\sigma}(1 + \frac{1}{p} + (1-\gamma)(1 + \frac{1}{r})) - 2a}{N} \sum_{k=1}^N { \alpha_{t}^{(k)}} e_t^{(k)} .
    \end{aligned}
\end{equation}

From the definitions in Theorems~\ref{thm:bcmf_convex},\ref{thm:strong_convex},
 \begin{equation}\label{eq:dist_def}
 \begin{aligned}
 \beta_3(p,r) = \Big( p  +(1-\gamma)(1+r) \Big) \\
 \tilde{a}_1(p,r) = 2/ \Big( \frac{1}{\sigma} + \frac{1}{\sigma p} + \frac{(1-\gamma)(1+ \frac{1}{r})}{\sigma} \Big) \\
 \tilde{a}_2(p,r) = \Big(2 a  - \frac{a^2 }{\sigma} - \frac{a^2 }{\sigma p}   - (1-\gamma)(1+ \frac{1}{r}) \frac{a^2 }{\sigma}  \Big)
 \end{aligned}     
 \end{equation}

From Lemma~\ref{lem:ex_bound_lemma}, $\exists\  (p(\epsilon),r(\epsilon))$ such that 

\begin{equation}
\begin{aligned}
\Big(  p(\epsilon)  +(1-\gamma)(1+r(\epsilon)) \Big) = 1 - \delta < 1 &&\\
\tilde{a}_1(p(\epsilon),r(\epsilon)) >\zeta -\epsilon. \\
\end{aligned}    
\end{equation}

Choosing $(p,r) = (p(\epsilon),r(\epsilon))$ in~\ref{eq:dist_def}, 
if $ a \leq  \zeta - \epsilon$, then $a \leq  \tilde{a}_1(p(\epsilon),r(\epsilon))$.

It is shown in the proof of Theorem~\ref{thm:strong_convex}, that  $a \leq  \tilde{a}_1(p(\epsilon),r(\epsilon))$ implies  $\tilde{a}_2(p(\epsilon),r(\epsilon))>0$. Also, $\alpha_t \geq \tilde{\alpha}_{\mathrm{min}}\rho$ and $e_t^{(k)} = f_{i_t}^{(k)}(x_t) - f_{i_t}^{(k)}(x^*)$. Substituting these values in~\eqref{eq:dist_p_r_subs},

\begin{equation}
    \begin{aligned}
        \|\hat{x}_{t+1} - x^*\|^2 + \sum_{k=1}^N \frac{1}{N} {||m_{t+1}^{(k)}||^2} &\leq \|\hat{x}_t - x^* \|^2 + \frac{1}{N} \sum_{k=1}^N {||m^{(k)}_t||^2} \\
        &- \frac{\tilde{a}_2(p(\epsilon),r(\epsilon))}{N} \sum_{k=1}^N { \tilde{\alpha}_{\mathrm{min}}\rho} (f_{i_t}^{(k)}(x_t) - f_{i_t}^{(k)}(x^*) ),
    \end{aligned}
\end{equation}

\begin{equation}
\begin{aligned}
    \frac{\alpha_{\mathrm{min}}\rho \big(\tilde{a}_2(p(\epsilon),r(\epsilon))\big)}{N} \sum_{k=1}^N (f_{i_t}^{(k)}(x_t) - f_{i_t}^{(k)}(x^*) ) &\leq \|\hat{x}_t - x^* \|^2 -  \|\hat{x}_{t+1} - x^*\|^2 &&\\
    &+ \frac{1}{N} \sum_{k=1}^N (||m^{(k)}_t||^2 - ||m_{t+1}^{(k)}||^2).
\end{aligned}    
\end{equation}

Taking expectation $E_{i_t}$ with respect to the sampled data points chosen at all nodes $k = 1, \ldots, N$ at time $t$ conditioned on the entire past, 

\begin{equation}
\begin{aligned}
    \alpha_{\mathrm{min}}\rho\big(\tilde{a}_2(p(\epsilon),r(\epsilon))\big (f(x_t) - f(x^*)) &\leq  \|\hat{x}_t - x^* \|^2 - E_{i_t}[\|\hat{x}_{t+1} - x^*\|^2] \\
    &+ \frac{1}{N} \sum_{k=1}^N (||m^{(k)}_t||^2 - E_{i_t}[||m_{t+1}^{(k)}||^2]).
\end{aligned}    
\end{equation}

Taking expectation $E$ over the whole process and averaging over all $t$, 

\begin{equation}
\begin{aligned}
    \alpha_{\mathrm{min}}\rho \big(\tilde{a}_2(p(\epsilon),r(\epsilon))\big) &\frac{1}{T} {\sum_{t=0}^{T-1} (E[f(x_t)] - E[f(x^*)]) }\\
    &\leq \frac{1}{T} \sum_{t=0}^{T-1} (E[\|\hat{x}_t - x^* \|^2] - E[\|\hat{x}_{t+1} - x^*\|^2]) \\
    &+ \frac{1}{NT} \sum_{k=1}^N \sum_{t=0}^{T-1} (E[||m^{(k)}_t||^2] - E[||m_{t+1}^{(k)}||^2]) \\
    &\leq \frac{1}{T} (E[\|\hat{x}_0 - x^* \|^2] - E[\|\hat{x}_{T} - x^*\|^2]) \\
    &+ \frac{1}{NT} \sum_{k=1}^N (E[||m^{(k)}_0||^2] - E[||m_{T}^{(k)}||^2]).
    \end{aligned}
\end{equation}

Using Jensen's inequality and $m_0^{(k)} = 0$ for all $k$, 
\begin{equation}
     \alpha_{\mathrm{min}}\rho\big(\tilde{a}_2(p(\epsilon),r(\epsilon))\big) \left(E[f(\frac{1}{T} {\sum_{t=0}^{T-1} x_t})] - E[f(x^*)]\right) \leq \frac{1}{T} \left( E[\|{x}_0 - x^* \|^2]  \right).
\end{equation}
Hence proved.
\end{proof}

\begin{remark}
Note that geometric convergence in the distributed setting for CSGD-ASSS can be obtained similarly from the single node case under the assumption of smoothness and that at least one function $f_{i}^{(k)}$ is strongly convex. 
\end{remark}

\section{Proof of Scaled Armijo Step-Size Search GD}\label{sec:det_gd}
A proof for $O(\frac{1}{T})$ convergence of Armijo step-size search GD~(CSGD-ASSS without the stochasticity in $f_{i_t}$) with scaling for all $\sigma \in (0,1)$ and convex function $f$ is presented in this section.
\begin{theorem}
(Deterministic-Uncompressed-Convex) Let $f(x)$ be convex and $L$ smooth. The Armijo step-size search gradient descent for $\sigma \in (0,1)$ and scale factor $a < 2 \sigma$ satisfies
\begin{equation}
\begin{aligned}
f(\Bar{x}_T) - f(x^*) \leq \frac{ ||x_0- x^*||^2  }{\tilde{\alpha}_{\mathrm{min}} \rho (2  a  - \frac{a^2 }{ \sigma}) T}
\end{aligned}    
\end{equation}

where $\Bar{x}_T =  \frac{1}{T} {\sum_{i=0}^{T-1} x_i}$.
\end{theorem}

\begin{proof}

The descent step is $x_{t+1} = x_t- \alpha_t a \nabla f(x_t)$ where $a$ is the scaling factor and $x^*$ is the minimum.

\begin{equation}
\begin{aligned}
||x_{t+1}-x^* ||^2 &= ||x_t - \alpha_t a \nabla f(x_t) -x^*||^2  \\
&= ||x_{t} - x^*||^2 + \alpha_t^2 a^2 ||\nabla f(x_t)||^2\\
&\hspace{1cm}- 2 \alpha_t a \langle x_t - x^*, \nabla f(x_t) \rangle 
\end{aligned}
\end{equation}

From the convexity of $f$,

\begin{equation}
      \langle  x^*- x_t, \nabla f(x_t) \rangle \leq f(x^*) - f(x_t).
\end{equation}

This implies 
\begin{equation}\label{eq:da_eq1}
\begin{aligned}
||x_{t+1}-x^* ||^2 &\leq ||x_t - x^*||^2 + \alpha_t^2 a^2 || \nabla f(x_t)||^2 \\
&\hspace{1cm}+ 2\alpha_t  a (f(x^*) - f(x_t )).\\
\end{aligned}    
\end{equation}

The non-scaled iterate $\tilde{x}_t$~($\tilde{x}_t = x_t - \alpha_t \nabla f(x_t)$) corresponding to the step-size $\alpha_t$ returned by the Armijo step-size search satisfies

\begin{equation}
\begin{aligned}
f(\tilde{x}_{t+1})- f(x_t) \leq -\alpha_t \sigma|| \nabla f(x_t)||^2.
\end{aligned}    
\end{equation}

This implies 

\begin{equation}\label{eq:da_eq2}
\begin{aligned}
\alpha_t^2 a^2  ||\nabla f(x_t)||^2 \leq \frac{a^2 \alpha_t}{ \sigma} (f(x_t)- f(\tilde{x}_{t+1})).
\end{aligned}    
\end{equation}

Substituting \eqref{eq:da_eq2} in \eqref{eq:da_eq1},
\begin{equation}
\begin{aligned}
||x_{t+1}-x^* ||^2 &\leq ||x_t - x^*||^2 +\frac{a^2 \alpha_t}{ \sigma} (f(x_t)- f(\tilde{x}_{t+1}))\\
&\hspace{1cm}+ 2\alpha_t  a (f( x^*) - f(x_t)).\\
\end{aligned}    
\end{equation}

Since $f(\tilde{x}_{t+1}) \geq f(x^*)$,
\begin{equation}
\begin{aligned}
||x_{t+1}-x^* ||^2 &\leq ||x_t - x^*||^2 +\frac{a^2 \alpha_t}{ \sigma} (f(x_t)- f(x^*))\\
& \hspace{1cm}+ 2\alpha_t  a (f(x^*) - f(x_t)),&&\\
\end{aligned}    
\end{equation}

\begin{equation}
\begin{aligned}
\alpha_t (2  a  - \frac{a^2 }{ \sigma}) ( f(x_t) - f(x^*) )  \leq ||x_t - x^*||^2 -  ||x_{t+1}-x^* ||^2 \\
\end{aligned}    
\end{equation}

{$ (2  a  - \frac{a^2 }{ \sigma}) > 0 \iff a<2\sigma $ }. \textit{Without scaling the factor $ (2  a  - \frac{a^2 }{ \sigma})$ would be $(2 - \frac{1}{\sigma})$, thus limiting $\sigma\in[0.5,1)$ in \cite{stoc_painless}}.

For Armijo step-size search, $\alpha_t \geq \frac{2(1-\sigma)}{L}\rho$. Using this lower bound, summing over the horizon $T$ and using Jensen's inequality,

\begin{equation}
\begin{aligned}
 f(\frac{1}{T}{\sum_{t=1}^{T} x_t}) - f(x^*)   \leq \frac{||x_0 - x^*||^2}{\frac{2(1-\sigma)}{L}\rho (2  a  - \frac{a^2 }{ \sigma}) T}.  \\
\end{aligned}    
\end{equation}
Hence proved.

\end{proof}

\section{Additional Simulation Results and Discussion}

\subsection{Validation Accuracy Comparison between Adaptive and Non-Adaptive Compressed SGD}
A summary of validation accuracy of our experiments on ResNets and DenseNets is presented in Table \ref{tab:val}, where NN, DS, CP, R34, R18, D121, C10, C100 refer to Neural Network, Data Set, Compression Percentage, ResNet 34, ResNet 18, DenseNet 121, CIFAR 10, and CIFAR 100 respectively. The columns $0.1,0.05,0.01$ state validation accuracy of non-adaptive compressed SGD \cite{aji_hea,stich2018sparsified} with step-sizes $0.1,0.05$ and $0.01$ respectively. The column $3\sigma$ corresponds to our proposed CSGD-ASSS algorithm with $a=3\sigma$. The compression percentages we have used are not fine-tuned, and are arbitrary. Since we do not compress layers with less than $1000$ parameters, consistent with~\cite{basu_qsparselocalsgd}, the batch normalization layers of DenseNet 121 are not compressed. Compressing the layers excluding the batch normalization layers at $1\%$ results in an average of $4\%$ per layer compression. From the table, we infer that, the validation accuracy of CSGD-ASSS is competitive with tuned non-adaptive compressed SGD.

\begin{table*}
\centering
\begin{tabular}{|c|c|c|c|c|c|c|c|}
\hline
\textbf{NN} & \textbf{DS}& \textbf{CP}  &  \textbf{0.1}&  \textbf{0.05}& \textbf{0.01} & $\mathbf{3 \sigma}$\\
\hline
R34& C100&{1.5}&72.0 & 71.39 & 70.27& 71.64\\
\hline
R18& C100&{1.5}&71.55 & 71.74 & 69.85& 72.82\\
\hline
R34& C10&{1.5}&92.42& 92.19& 91.16& 90.42\\
\hline
R18& C10&{1.5}&92.34& 92.23& 91.28& 91.1\\
\hline
R34& C100&{10}&71.81& 71.21& 69.24& 72.55\\
\hline
R18& C100&{10}&71.64& 71.79& 70.04& 72.46\\
\hline
R34& C10&{10}&92.58& 92.44& 91.42& 91.05\\
\hline
R18& C10&{10}&92.28& 92.1& 91.11& 91.46\\
\hline
D121& C100&{4}&69.18& 68.45& 64.64& 71.63\\
\hline
D121& C100&{10}&68.65& 67.91& 63.26& 70.82\\
\hline
D121& C10&{4}&91.09& 91.06& 89.09& 90.89\\
\hline
D121& C10&{10}&91.39& 90.8& 89.42& 91.25\\
\hline
\end{tabular}
\captionsetup{justification=centering}
\caption{{
Validation accuracy of ResNets, DenseNet \\ 
NN- Neural Network, DS - Dataset, CP- Compression Percentage, R34 - ResNet 34, R18 - ResNet18, D121- DenseNet 121, Non-adaptive step-sizes - 0.1,0.05,0.01, CSGD-ASSS - 3$\sigma$}}
\label{tab:val}
\end{table*}

\subsection{Effect of Scaling on Armijo Step-Size Search}
\begin{figure*}%
\centering
\begin{subfigure}{.49\columnwidth}
\includegraphics[width=\columnwidth]{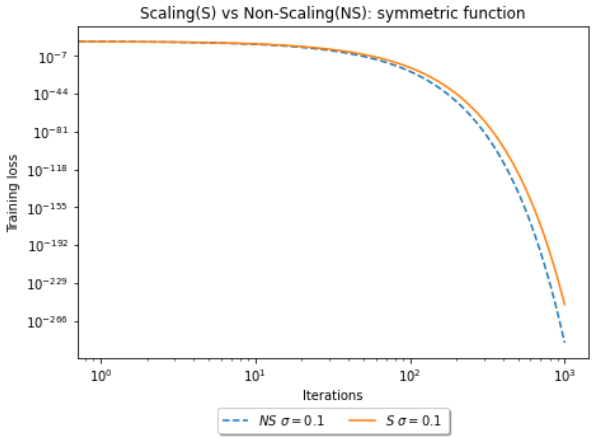}%
\caption{{ GD on $\sum_{i=1}^{10} \frac{x_i^2}{2^5}$}}%
\label{fig:gd_symmetric}%
\end{subfigure}\hfill%
\begin{subfigure}{.49\columnwidth}
\includegraphics[width=\columnwidth]{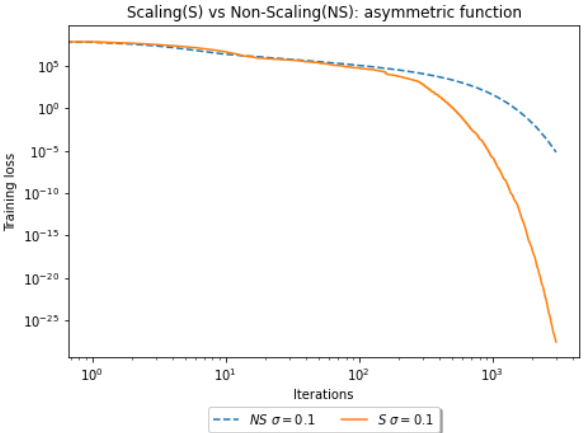}%
\caption{{ GD on $\sum_{i=1}^{10} \frac{x_i^2}{2^i}$} }%
\label{fig:gd_asymmetric}%
\end{subfigure}%
\caption{Scaled~(S $\sigma =0.1$) vs non-scaled~(NS $\sigma=0.1$) GD at $\sigma = 0.1$ and $a = 1.5 \sigma$}
\label{fig:scale_plots}
\end{figure*}

To demonstrate the effect of scaling on GD with Armijo step-size search, we plot scaled and non-scaled Armijo step-size search on the symmetric curve $\sum_{i=1}^{10} \frac{x_i^2}{2^5}$ and the asymmetric curve  $\sum_{i=1}^{10} \frac{x_i^2}{2^i}$ in Figures~\ref{fig:gd_symmetric} and \ref{fig:gd_asymmetric}, respectively. We set the parameters of the step-size search to be $\sigma =0.1$ and scale value $a = 1.5 \sigma$. On the symmetric curve, both the scaled and non-scaled step-size search algorithms have comparable performance, however, on the asymmetric curve, the scaled algorithm outperforms the non-scaled algorithm by several orders of magnitude. Machine learning tasks can have loss functions that are asymmetric, and hence scaling can prove to be a useful technique for adaptive step-size methods.

\begin{figure*}%
\centering
\begin{subfigure}{.49\columnwidth}
\includegraphics[width=\columnwidth]{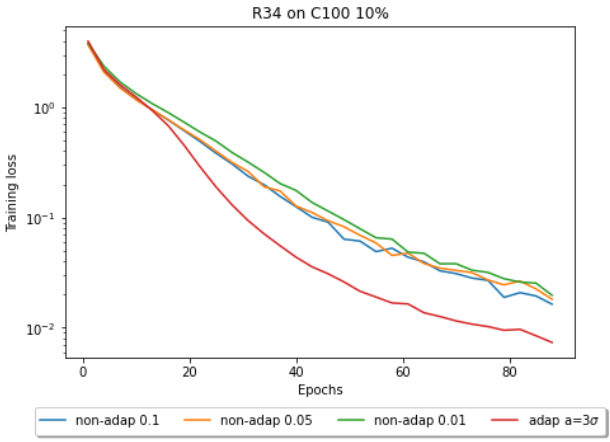}%
\caption{{ ResNet 34, CIFAR100, $\approx 10\%$} }%
\label{fig:r34_c100_10per}%
\end{subfigure}\hfill%
\begin{subfigure}{.49\columnwidth}
\includegraphics[width=\columnwidth]{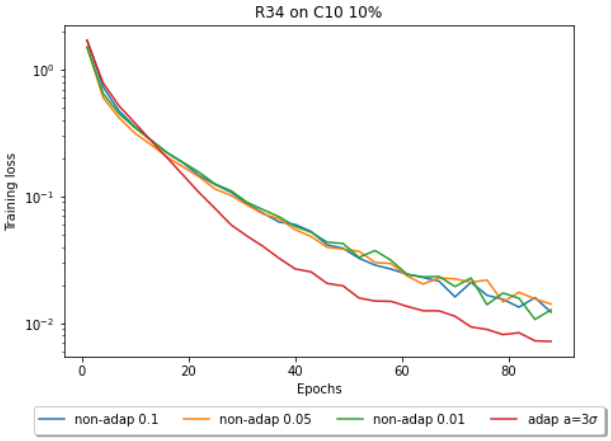}%
\caption{{ ResNet 34, CIFAR10, $\approx 10\%$}}%
\label{fig:r34_c10_10per}%
\end{subfigure}%

\caption{Training loss of ResNet 34 on CIFAR100 and CIFAR10 at $\approx10\%$ compression}
\label{fig:4}
\end{figure*}

\begin{figure*}%
\centering
\begin{subfigure}{.49\columnwidth}
\includegraphics[width=\columnwidth]{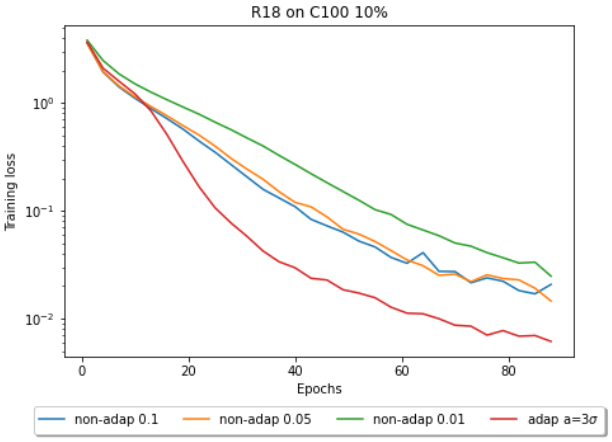}%
\caption{{ ResNet 18, CIFAR100, $\approx 10\%$} }%
\label{fig:r18_c100_10per}%
\end{subfigure}\hfill%
\begin{subfigure}{.49\columnwidth}
\includegraphics[width=\columnwidth]{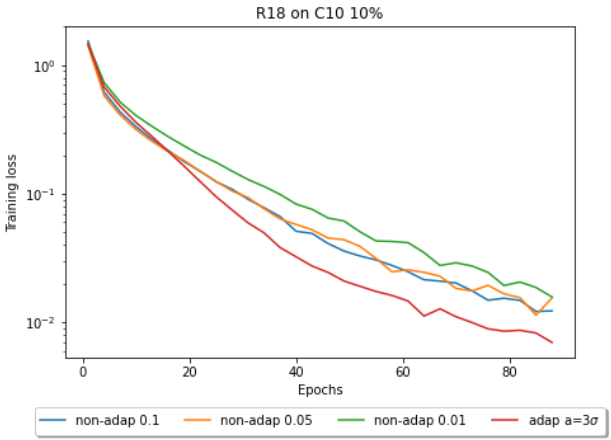}%
\caption{{ ResNet 18, CIFAR10, $\approx 10\%$}}%
\label{fig:r18_c10_10per}%
\end{subfigure}%

\caption{Training loss of ResNet 18 on CIFAR100 and CIFAR10 at $\approx10\%$ compression}
\label{fig:5}
\end{figure*}

\end{document}